\documentclass[10pt]{article}
\usepackage{hyperref}
\usepackage{url}
\usepackage{booktabs}
\usepackage{amsfonts}
\usepackage{nicefrac}
\usepackage{microtype}
\usepackage{wrapfig}
\usepackage{color}
\usepackage{authblk}
\usepackage{subfigure}
\usepackage{booktabs}

\usepackage{todonotes}   

\usepackage{enumitem}
\usepackage{graphicx}
\usepackage{amsmath,amssymb}
\usepackage{amsthm}
\usepackage{bigstrut}
\usepackage{float}
\usepackage{caption}
\usepackage{mathtools}
\usepackage{algorithm}

\newtheorem{theorem}{Theorem}
\newtheorem*{definition}{Definition}
\newtheorem{remark}{Remark}
\newtheorem{lemma}{Lemma}

\DeclareMathOperator*{\FRC}{FRC}
\DeclareMathOperator*{\unc}{Unc}
\DeclareMathOperator*{\egc}{EGC}
\DeclareMathOperator*{\agc}{AGC}
\DeclareMathOperator*{\tot}{Total}

\newcommand{\PP}{\mathbb{P}}
\newcommand{\EE}{\mathbb{E}}

\newcommand{\real}{\mathbb{R}}

\newcommand{\ones}{{\bf 1}}

\usepackage{algorithmic}
\usepackage[numbers]{natbib}

\usepackage{geometry}
\newcommand{\erasegd}{\textsc{ErasureHead}}

\title{\erasegd: Distributed Gradient Descent without Delays Using Approximate Gradient Coding}

\author[1,2]{{Hongyi Wang}}
\author[2]{{Zachary Charles}}
\author[2]{{Dimitris Papailiopoulos}}

\affil[1]{Department of Computer Sciences, UW--Madison}
\affil[2]{Department of Electrical and Computer Engineering, UW--Madison}

\begin{document}
	
	\maketitle
	
	\begin{abstract}
	We present \erasegd, a new approach for distributed gradient descent (GD) that mitigates system delays by employing approximate gradient coding.
	Gradient coded distributed GD uses redundancy to {\it exactly} recover the gradient at each iteration from a subset of compute nodes. 
	\erasegd{} instead uses {\it approximate} gradient codes to recover an inexact gradient at each iteration, but with higher delay tolerance.
	Unlike prior work on gradient coding, we provide a performance analysis that combines both delay and convergence guarantees.
	We establish that down to a small noise floor, \erasegd{} converges as quickly as distributed GD and has faster overall runtime under a probabilistic delay model.
	We conduct extensive experiments on real world datasets and distributed clusters and demonstrate that our method can lead to significant speedups over both standard and gradient coded GD.
	\end{abstract}
	
	\section{Introduction}

	Much of the recent empirical success of machine learning has been enabled by distributed computation. In order to contend with the size and scale of modern data and models, many production-scale machine learning solutions employ distributed training methods.
	Ideally, distributed implementations of learning algorithms lead to speedup gains that scale linearly with the number of compute nodes.
	Unfortunately, in practice these gains fall short of what is theoretically possible even with a small number of compute nodes. Several studies, starting with \cite{dean2012large} extending to more recent ones \cite{qi17paleo,grubic2018synchronous}, have consistently reported a tremendous gap between ideal and realizable speedup gains.
	
	\begin{figure}[t]
	\centering
	\includegraphics[width=0.45\linewidth]{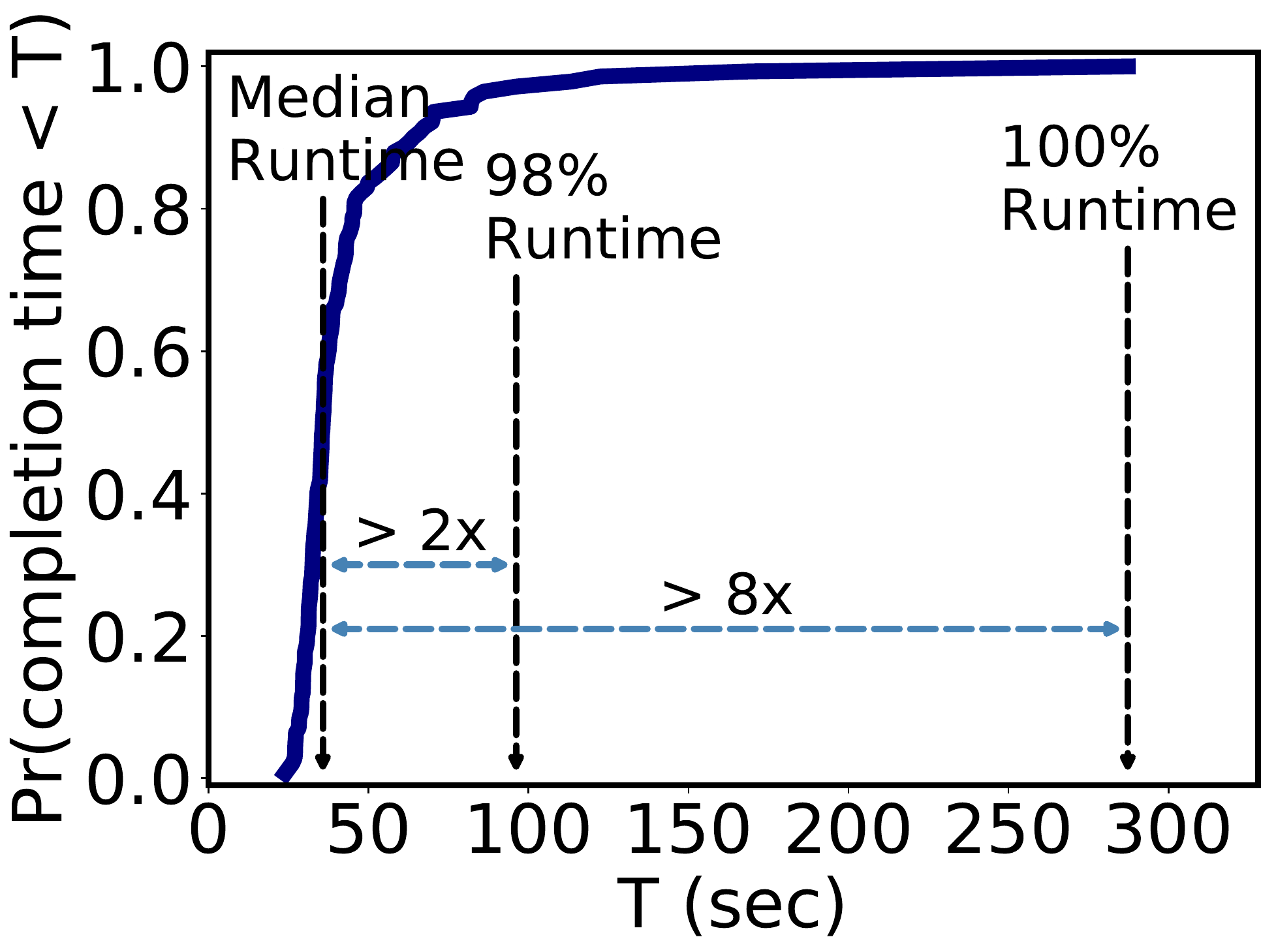}
	\vspace{-0.3cm}
	\caption{Probability of job completion time per worker. Setup: 148 worker nodes on AWS EC2 \texttt{t2.small} instances, running distributed SGD with batch size 1024 on CIFAR-10, and a simple variation of the cuda-covnet model. Observe that 100\% completion time corresponds to $8\times$ the median.}
	\label{fig:RunningExample}
	\vspace{-0.5cm}
\end{figure}

	There are many causes of this phenomenon. A notable and well-studied one is the presence of communication bottlenecks~\cite{dean2012large, seide20141, strom2015scalable, qi17paleo,alistarh2017qsgd,grubic2018synchronous,wang2018atomo}, while another is the presence of {\it straggler} nodes. These are compute nodes whose runtime is significantly slower than the average node in the system. This {\it straggler} effect can be observed practically in many real world distributed machine learning applications. Figure \ref{fig:RunningExample} illustrates one running example of {\it straggler} effect. 
	
	In this work, we focus on the latter. Stragglers become a bottleneck when running {\it synchronous} distributed algorithms that require explicit synchronization between tasks. 
	This is the case for commonly used synchronous optimization methods, 
	including mini-batch stochastic gradient descent (SGD), and gradient descent (GD), where the overall runtime is determined by the slowest response time of the distributed tasks.  
	
	Straggler mitigation for general workloads has received significant attention from the systems community, with  techniques ranging from job replication approaches, to predictive job allocation, straggler detection, and straggler dropping mechanisms \cite{ananthanarayanan2013effective, zaharia2008improving, chen2016revisiting}. 
	In the context of model training, one could use asynchronous techniques such as \textsc{Hogwild!} \cite{recht2011hogwild}, but these may suffer from reproducibility issues due to system randomness, making them potentially less desirable in some production environments \cite{chen2016revisiting}.

	Recently, coding theory has provided a popular tool set for mitigating the effects of stragglers. 
	Codes have recently been used in the context of machine learning, and applied to problems such as data shuffling \cite{li2015coded,lee2016speeding}, distributed matrix-vector and matrix-matrix multiplication \cite{lee2016speeding,dutta2016short}, as well as distributed training \cite{li2018near, yu2017polynomial, karakus2017encoded, fahim2017optimal, park2018hierarchical, chen2018draco, yu2018straggler}.

	In the context of distributed, gradient-based algorithms, Tandon et al. \cite{tandon2017gradient}  introduced {\it gradient coding} as a means to mitigate straggler delays. They show that with $n$ compute nodes, one can assign $c$ gradients to compute node such that the sum of the $n$ gradients can be recovered from any $n-c+1$ compute nodes. Unfortunately, for $c$ small this may require waiting for almost all of the compute nodes to finish, while for larger $c$ the compute time per worker may outweigh any benefits of straggler mitigation. 
	
	A reasonable alternative to these exact gradient codes are {\it approximate} gradient codes, where we only recover an approximate sum of gradients. This seems to be a sensible compromise, as first order methods are known to be be robust to small amounts of noise \cite{mania2017perturbed}. Approximate gradient codes (AGCs) were first analyzed in \cite{raviv2017gradient} and \cite{charles2017approximate}. While these studies showed that AGCs can provide significantly higher delay tolerance when recovering the sum of gradients within a small error, neither work provided a rigorous performance analysis or a practical implementation of AGCs on any training algorithm. 

	\subsection{Our Contributions}

		In this work, we present \erasegd{}, an end-to-end distributed training method that implements the approximate gradient codes in studied in \cite{charles2017approximate}. The theoretical novelty of this work is that we rigorously analyze the convergence and delay properties of GD in the presence of stragglers when using approximate gradient codes. The practical novelty is that \erasegd{} allows for ``erasures'' of up to a constant fraction of all compute nodes, while introducing small noise in the recovered gradient, which leads to a significantly reduced end-to-end run time compared to ``vanilla" and gradient coded GD.

		\paragraph{Theoretical results.} We analyze the performance of \erasegd{} by analyzing the convergence rate of distributed gradient descent when using AGCs to compute the first-order updates. We focus on functions satisfying the Polyak-\L{}ojasiewicz (PL) condition, which generalizes the notion of strong convexity. While full-batch gradient descent achieves linear convergence rates on such functions, it is well-known that SGD only achieves a convergence rate of $O(1/T)$. Despite the fact that \erasegd{} uses a stochastic gradient type of an update, we show that it achieves a linear convergence rate up to a relatively small noise floor. 
		
		Specifically, suppose $f$ is $\mu$-PL, $\beta$-smooth and we use \erasegd{} to minimize $f$, where we initialize at $x_0$ and at each iteration update $x_t$ using the output of \erasegd{}. If we assign each compute node $c$ gradients and wait for a $\delta$ fraction of the nodes to finish, then we have the following convergence rate:
		$$\Delta_T \leq \left(1-\frac{(1-e^{-c\delta})\mu}{\beta}\right)^T\Delta_0 + \tilde{O}\left(\frac{e^{-c\delta}}{n}\right)$$
		where $\Delta_T = \EE[f(x_T)-f^*]$, $f^* = \inf_{x}f(x)$, and $\tilde{O}$ hides problem dependent variables related to $\mu$ and $\beta$. This result is stated formally in Theorem \ref{thm:conv_rate_2}.

		This result does not capture the end-to-end runtime, as it does not model stragglers or the fact that for larger $c$, each worker requires more compute time.
		In order to analyze the total runtime of \erasegd, we use a probabilistic model for stragglers proposed in \cite{lee2016speeding}. We show that if we apply distributed GD, gradient coded GD, and \erasegd{} on $n$ compute nodes, they have total runtimes that are approximately $O(T\cdot \log(n)/n)$, $  O(T\cdot\log(n/c)/n)$, and $\tilde{O}(T/n)$, where $T$ is the total iterations to reach to an accuracy $\Delta_T = \tilde{O}(e^{-c}/n)$. Thus, \erasegd{} can lead to almost a $\log(n)$ speedup over vanilla and gradient coded GD. This is made formal in Theorem \ref{thm:tot_time}.

		\paragraph{Experimental results.} Finally, we provide an extensive empirical analysis of \erasegd. We compare this to exact gradient coded gradient descent and uncoded gradient descent. Our results generally show that approximate gradient codes lead to up to $6\times$ and $3\times$ faster distributed training over vanilla and coded GD respectively. Moreover, we see these speedups consistently across multiple datasets and classification tasks. Our implementation is publicly available for reproducibility \footnote{https://github.com/hwang595/ErasureHead}.

	\subsection{Related Work}

		Many different works over the past few years have employed coding--theoretic ideas to improve the performance of various distributed algorithms. Prior work proposes different methods for reducing the effect of stragglers, including replicating jobs across nodes \cite{shah2016redundant} and dropping straggler nodes \cite{ananthanarayanan2013effective}. More recently, coding theory has become popular for mitigating the straggler effect in distributed machine learning. In particular, \cite{lee2016speeding} proposed the use of erasure codes for speeding up the computation of linear functions in distributed learning systems. Since then, many other works have analyzed the use of coding theory for distributed tasks with linear structure \cite{fahim2017optimal, park2018hierarchical, leecoded, dutta2016short, yu2017polynomial}.

		\cite{tandon2017gradient} proposed the use of coding theory for nonlinear machine learning tasks. They analyzed so-called gradient codes for distributed training, and showed that their gradient codes achieve the optimal trade-off between computation load and straggler tolerance. Other gradient codes that achieve this optimal trade-off have since been proposed \cite{raviv2017gradient,halbawi2018improving}. \cite{li2018near} used the Batched Coupon Collection problem to improve distributed gradient descent under a probabilistic straggler model. More sophisticated gradient codes that focus on the trade-off between computation and communication were introduced and analyzed in \cite{ye2018communication}.

		The aforementioned gradient codes generally focus on recovering the exact gradient in each iteration of gradient descent. However, \cite{raviv2017gradient, charles2017approximate, karakus2017encoded, charles2018gradient} utilize the fact that for distributed learning problems, approximations to the gradient may be acceptable. While \cite{karakus2017encoded} focuses on loss related to least-squares problems, \cite{raviv2017gradient, charles2017approximate, charles2018gradient} construct and analyze gradient codes that can be used in non-linear settings. \cite{raviv2017gradient} uses expander graphs to construct approximate gradient codes with small error in the worst-case straggler setting, \cite{charles2017approximate} focuses on the setting where the stragglers are chosen randomly. While the aforementioned work shows that these codes are effective from in terms of their $\ell_2$ error, they lack careful convergence rate analyses or an extensive experimental evaluation. 
	\section{Problem Setup}

	Suppose we wish to minimize $f(x) = \frac{1}{n}\sum_{i=1}^n f_i(x)$
	in a distributed manner. We focus on the setting where there are $k$ compute nodes and one master node, often referred to as the parameter server, which maintains a copy of $x$. In order to apply gradient descent, we would typically need to compute $\nabla f_i(x)$ for all $i$. In a standard distributed setting, the master node would allocate the $n$ gradient computations among the $k$ nodes, and each would compute their assigned gradients. Once the compute nodes complete the local computation, they send their outputs back to the parameter server, which sums them up and updates $x$.

	Unfortunately, while most compute nodes may take roughly the same amount of time to compute their allocated gradients, the overall runtime is hindered by the fact that we have to wait for the slowest node in the system, which may be much slower than average. This is referred to as the {\it straggler effect}. In order to avoid waiting for all compute nodes to finish, we can borrow ideas from coding theory to ensure that we can compute $\nabla f(x)$ from subsets of the $k$ compute nodes. To achieve that, \cite{tandon2017gradient} introduce {\it gradient coding}.

	We define gradient codes via their {\it function assignment matrix}. Given $n$ tasks and $k$ workers, a function assignment matrix $G$ is a $n\times k$ matrix whose $(i,j)$ entry is non-zero if worker $j$ computes task $i$. In our setting, the $n$ tasks are the $n$ gradients composing $\nabla f(x)$. If column $j$ of $G$ has support $U_j$, then node $j$ has output $y_j = \sum_{i \in U_j} G_{i,j}\nabla f_i(x)$, which it sends to the parameter server. The parameter server takes a linear combination of the outputs of non-straggler nodes, and uses this as a gradient update. Each iteration of our distributed training algorithm performs this procedure once. Given some function assignment matrix $G$, we assume that each column has $c$ non-zero entries. Thus, $G$ incurs a computational load $c$ per compute node. For simplicity, we assume $c$ divides $n$.

	We will analyze the fractional repetition code (FRC) from \cite{tandon2017gradient}. While this code was first proposed for exact gradient computations, it was used to for approximate gradient computations in \cite{charles2017approximate}. Given $n, k, c$ as above, let $\ell = kc/n$. The code $\FRC(n,k,c)$ is defined by the following function assignment matrix:
	$$G = \begin{pmatrix} \ones_{c\times \ell} & {\bf 0}_{c\times \ell} & {\bf 0}_{c\times \ell} & \ldots & {\bf 0}_{c\times \ell}\\ {\bf 0}_{c\times \ell} & \ones_{c\times \ell} & {\bf 0}_{c\times \ell} & \ldots & {\bf 0}_{c\times \ell}\\
	{\bf 0}_{c\times \ell} & {\bf 0}_{c\times \ell} & \ones_{c\times \ell} & \ldots & {\bf 0}_{c\times \ell}\\
	\vdots & \vdots & \vdots & \ddots & \vdots\\
	{\bf 0}_{c\times \ell} & {\bf 0}_{c\times \ell} & {\bf 0}_{c\times \ell} & \ldots & \ones_{c\times \ell}\end{pmatrix}.$$
	Here, $\ones_{c\times \ell}$ and ${\bf 0}_{c\times \ell}$ denote the $c$ by $\ell$ matrices of ones and zeros, respectively. $\FRC(n,k,c)$ computes $n$ tasks with $k$ workers, has computational load $c$ per worker, and has repetition factor $\ell = kc/n$.

	Due to the straggler effect, we only have access to the output of $r < k$ of the compute nodes. These are the {\it non-straggler nodes}. The output of the $\FRC(n,k,c)$ code is then
	\begin{equation}\label{eq:g_frc}
	g(x) = \frac{1}{n}\sum_{i=1}^{n/c}Y_ig[i](x)\end{equation}
	where $g[i](x)$ is the sum of the $i$-th block of gradients, that is,
	$g[i](x) = \sum_{j=1}^{n/c}\nabla f_{c(i-1)+j}(x).$
	Above, $Y_i$ is a Bernoulli random variable denoting whether or not there is a non-straggler column in the $i$-th block of $G$.

	\cite{tandon2017gradient} shows that if $r \geq k-c+1$, then no matter which nodes are stragglers, $g(x) = \nabla f(x)$. We refer to gradient codes where we wait until we compute the true gradient as {\it exact} gradient codes. When $c$ is small, this may be an unreasonable amount of time to wait. We may wish to only have $r$ be some constant fraction of $k$. If $r$ is small, $g(x)$ may not equal the true gradient $\nabla f(x)$. We refer to these as {\it approximate} gradient codes. There are trade-offs between the accuracy of our computation and the tolerance to stragglers. When the straggler effect is more pronounced, approximate gradient codes are potentially much faster than exact gradient codes. 

	\erasegd{} is an end-to-end distributed training system where instead of using the true gradient $\nabla f(x) = \frac{1}{n}\sum_{i=1}^n \nabla f_i(x)$ to update, we use the output of a $\FRC(n,k,c)$ approximate gradient code. In other words, \erasegd~uses a first-order update of the form
	\begin{equation}\label{eq:erase_update}
	x_{t+1} = x_t-\gamma g(x_t)
	\end{equation}
	where $g(x)$ is as in (\ref{eq:g_frc}). Here $\gamma$ is the step-size. We present \erasegd{} in algorithmic form below.

	\begin{algorithm}[tb]
		\caption{\erasegd: Master node's protocol}
		\label{alg:erase_master}
		\begin{algorithmic}
			\STATE {\bfseries Input:} $\{f_i\}_{i=1}^n$, $\gamma > 0$, $c > 0$, $\delta \in (0,1)$, $T > 0$
			\STATE Partition functions in to groups $C_i$ of size $c$
			\STATE Partition workers in to groups $W_i$ of size $kc/n$
			\STATE For each $i$, transmit $C_i$ to each worker in $W_i$
			\STATE Initialize $x$ randomly
			\FOR{$t=1$ {\bfseries to} $T$}
			\STATE Transmit $x$ to all workers
			\STATE $g \gets 0$, $S_m \gets \emptyset$ for $m \in [n/c]$
			\REPEAT 
				\STATE Receive $y_j$ from worker $j$ in $W_m$
				\IF{$S_m = \emptyset$}
				\STATE $g \gets g + y_j$, $S_m \gets \{j\}$
				\ENDIF
			\UNTIL{$\delta n$ workers finish or $\forall m,~W_m \cap S_m \neq \emptyset$ }
			\STATE $x\gets x-\gamma g/n$
			\ENDFOR
		\end{algorithmic}
	\end{algorithm}

	\begin{algorithm}[tb]
		\caption{\erasegd: Worker node's protocol}
		\label{alg:erase_worker}
		\begin{algorithmic}
			\STATE Receive $C$ from master node
			\REPEAT
				\STATE Receive $x$ from master node
				\STATE $y_j \gets \sum_{f \in C} \nabla f(x)$
				\STATE Send $y_j$ to master node
			\UNTIL{master node terminates}
		\end{algorithmic}
	\end{algorithm}

	In order to analyze the convergence rates, we will assume that $f$ is $\beta$-smooth and satisfies the Polyak-\L{}ojasiewicz (PL) condition.

	\begin{definition}A function $f$ is $\beta$-smooth if for all $x, y$,
	\begin{equation}\label{eq:beta_smooth}
	f(y) \leq f(x) + \langle f(x),y-x\rangle + \dfrac{\beta}{2}\|y-x\|^2.\end{equation}\end{definition}

	\begin{definition}A function $f$ is $\mu$-PL if it has a non-empty set of global minimizers $\mathcal{X}^*$ with minimal function value $f^*$ such that for all $x$,
	\begin{equation}\label{eq:pl}
	\dfrac{1}{2}\|\nabla f(x)\|^2 \geq \mu(f(x)-f^*).\end{equation}\end{definition}
	The PL condition is a generalization of strong convexity, as $\mu$-strong convexity implies $\mu$-PL. However, the PL condition does not require convexity. There are simple functions (such as $f(x) = x^2+\sin^2(x))$ that are not strongly convex, but are $\mu$-PL. Despite this potential non-convexity, many first-order methods have comparable convergence rates on PL and strongly convex functions \cite{karimi2016linear}. For example, gradient descent has a linear convergence rate on PL functions. We have the following lemma about the convergence of gradient descent on PL functions.

	\begin{lemma}[\cite{karimi2016linear}]\label{lem:conv_pl}Suppose that $f$ is $\mu$-PL and $\beta$-smooth. Then gradient descent on $f$ with step-size $\gamma = 1/\beta$ has convergence rate
	$$f(x_T)-f^* \leq \left(1-\frac{\mu}{\beta}\right)^T[f(x_0)-f^*].$$\end{lemma}
	\section{Convergence Analysis}\label{sec:conv}

	In this section we will analyze the convergence rate of \erasegd. Omitted proofs can be found in the appendix. We will do this in two steps. First, we bound the first and second moments of the vector $g$ defined in \eqref{eq:g_frc}. We then combine these moment bounds with techniques from \cite{karimi2016linear} to bound the convergence rate of our method.

	We first wish to understand the random variables $Y_i$ in \eqref{eq:g_frc}. We first prove the following lemma concerning their moments.

	\begin{lemma}\label{lem:p_q_comp}Let $i, j \in [n/c]$ with $i \neq j$. Suppose the $r$ non-stragglers are selected randomly. Then the $Y_i$ satisfy:
	\begin{enumerate}
		\item $\EE[Y_i] = 1-\dfrac{\binom{k-\ell}{r}}{\binom{k}{r}}$.
		\item $\EE[Y_iY_j] = 1-\dfrac{2\binom{k-\ell}{r} - \binom{k-2\ell}{r}}{\binom{k}{r}}.$
	\end{enumerate}
	\end{lemma}

	\begin{proof}
		Clearly, $Y_i$ is 0 iff all $r$ non-stragglers are selected from the $k-\ell$ columns of $G$ that are not from the $i$-th block. Therefore, $Y_i$ is Bernoulli with probability $1-p$ where $p = \binom{k-\ell}{r}/\binom{k}{r}$.

		Similarly, $Y_iY_j$ is 0 iff the $i$-block or the $j$-th block have no non-stragglers. Denote these events as $A_i, A_j$ respectively. Then by inclusion-exclusion,
		\begin{align*}
			\PP[A_i \cup A_j] = \PP[A_i] + \PP[A_j] - \PP[A_i \cap A_j] = 2p - \dfrac{\binom{k-2\ell}{r}}{\binom{k}{r}}.
		\end{align*}
	\end{proof}	

	We define $p$ such that $\EE[Y_i] = 1-p, \EE[Y_iY_j] = 1-q$. Simple analysis shows that $p \leq q \leq 2p$. We then have the following lemma concerning the first two moments of $g(x)$ defined in \eqref{eq:g_frc}.

	\begin{lemma}\label{thm:g_moments}The random variable $g(x)$ satisfies
	\begin{enumerate}
		\item {\small $\EE[g(x)] = (1-p)\nabla f(x)$.}
		\item {\small $\EE[\|g(x)\|^2] = (1-q)\|\nabla f(x)\|^2 + \frac{q-p}{n^2}\sum_{i=1}^{n/c} \|g[i](x)\|^2$.}
	\end{enumerate}
	Here, the expectation is taken only with respect to the set of non-stragglers.
	\end{lemma}
	\begin{proof}
		Recall that from \eqref{eq:g_frc} we have
		$$g(x) = \dfrac{1}{n}\sum_{i=1}^{n/c}Y_ig[i](x)$$
		where $g[i](x)$ is the sum of the $i$-th block of gradients. Since $\EE[Y_i] = 1-p$, this immediately implies condition 1.

		For condition 2, we have

		\begin{align*}
			\EE[\|g(x)\|^2] &= \frac{1}{n^2}\EE\langle \sum_{i=1}^{n/c} Y_ig[i](x), \sum_{i=1}^{n/c} Y_ig[i](x)\rangle\\
			&=\frac{1}{n^2} \left(\sum_{i=1}^{n/c} \EE[Y_i^2]\|g[i](x)\|^2 + \sum_{i \neq j} \EE[Y_iY_j]\langle g[i](x),g[j](x)\rangle\right)\\
			&= \frac{1}{n^2} \left( (1-p)\sum_{i=1}^{n/c} \|g[i](x)\|^2 + (1-q)\sum_{i \neq j} \langle g[i](x),g[j](x)\rangle\right)\\
			&= \frac{(1-q)}{n^2}\sum_{i,j}\langle g[i](x),g[j](x)\rangle + \frac{q-p}{n^2}\sum_{i=1}^{n/c} \|g[i](x)\|^2\\
			&= (1-q)\|\nabla f(x)\|^2 + \frac{q-p}{n^2}\sum_{i=1}^{n/c} \|g[i](x)\|^2.
		\end{align*}
	\end{proof}	

	If for all $i$ and $x$, we have $\|\nabla f_i(x)\| \leq \sigma$, we get the following bound.
	\begin{equation}\label{eq:gi_bound}
	\|g[i](x)\|^2 \leq \frac{1}{n^2}\left(\sum_{j=1}^c\|\nabla f_{c(i-1)+j}(x)\|\right)^2 \leq \frac{c^2\sigma^2}{n^2}.\end{equation}

	Combining \eqref{eq:gi_bound} and Lemma \ref{thm:g_moments} we get the following lemma.
	\begin{lemma}If for all $i$ and $x$, $\|\nabla f_i(x)\| \leq \sigma$, then
	$$\EE[\|g(x)\|^2] \leq (1-q)\|\nabla f(x)\|^2 + \dfrac{(q-p)c\sigma^2}{n}.$$
	\end{lemma}

	Define $\hat{g}(x) := g(x)/(1-p)$. By Theorem \ref{thm:g_moments}, $\EE[\hat{g}(x)] = \nabla f(x)$. We perform gradient descent using the update
	$$x_{t+1} = x_t - \gamma \hat{g}(x)$$
	where $\hat{g}$ is computed after viewing the output of $r$ non-stragglers. We refer to this as $(n,k,c,r)$-AGC gradient update (approximate gradient code).

	\begin{remark}One popular method for distributed gradient descent is to partition the $n$ tasks to the $k$ workers and wait for all but $m$ of the compute nodes to finish. This is sometimes referred to as allowing $m$ erasures. This is equivalent to setting $c = \frac{n}{k}$ and $r = k-m$, which is $(n,k,\frac{n}{k},k-m)$-AGC.\end{remark}

	Define $\Delta_T = \EE[f(x_{T})-f^*]$. We can then show the following theorem about the convergence rate of our method.
	\begin{theorem}\label{thm:conv_rate}Suppose $f$ is $\beta$-smooth and $\mu$-PL with $\beta \geq \mu > 0$ and that for all $i \in [n]$ and $x$, $\|\nabla f_i(x)\| \leq \sigma$.
	If $\gamma = 1/\beta$, then
	$$\Delta_T \leq \left(1-\frac{\mu}{\beta}\right)^T\Delta_0 + \frac{\sigma^2}{2(1-p)\mu}\left(p + \frac{(q-p)c}{(1-p)n}\right).$$
	If $\gamma = (1-p)/\beta$ then
	$$\Delta_T \leq \left(1-\frac{(1-p)\mu}{\beta}\right)^T\Delta_0 +\frac{(q-p)c\sigma^2}{2(1-p)\mu n}.$$
	\end{theorem}

	\begin{proof}
		We first prove (1). By $\beta$-smoothness, we have

		\begin{equation}\label{eq:g_beta}
		f(x_{t+1}) \leq f(x_t) - \gamma \langle \nabla f(x_t), \hat{g}(x_t)\rangle + \dfrac{\beta\gamma^2}{2}\|\hat{g}(x_t)\|^2.\end{equation}

		Taking an expectation with respect to the random set of non-stragglers at iteration $t$ and applying Theorem \ref{thm:g_moments},

		\begin{align*}
		\EE[f(x_{t+1})] \leq & \; f(x_t) - \gamma\|\nabla f(x_t)\|^2 + \dfrac{\beta\gamma^2(1-q)}{2(1-p)^2}\|\nabla f(x_t)\|^2 + \dfrac{\beta\gamma^2(q-p)c\sigma^2}{2(1-p)^2n}\\
		\leq & \; f(x_t) - \gamma\|\nabla f(x_t)\|^2 + \dfrac{\beta\gamma^2}{2(1-p)}\|\nabla f(x_t)\|^2 + \dfrac{(q-p)\beta\gamma^2c\sigma^2}{2(1-p)^2n}\\
		= & \; f(x_t) + \left(\dfrac{\beta\gamma^2}{2}-\gamma\right)\|\nabla f(x_t)\|^2 + \dfrac{p\beta\gamma^2\sigma^2}{2(1-p)} + \dfrac{(q-p)\beta\gamma^2c\sigma^2}{2(1-p)^2n}\\
		\end{align*}

		This last step used the fact that if $\|\nabla f_i(x)\| \leq L$ for all $i$, then $\|\nabla f(x)\| \leq \sigma$. Using the step-size $\gamma = 1/\beta$, we have

		\begin{align*}
		\EE[f(x_{t+1})] & \leq f(x_t) - \dfrac{1}{2\beta}\|\nabla f(x_t)\|^2 + \dfrac{p\sigma^2}{2(1-p)\beta} + \dfrac{(q-p)c\sigma^2}{2(1-p)^2\beta n}.
		\end{align*}

		Let $\kappa = \mu/\beta$ denote the condition number and define $\psi$ as
		$$\psi := \dfrac{p\sigma^2}{2(1-p)\beta} + \dfrac{(q-p)c\sigma^2}{2(1-p)^2\beta n}.$$
		Using the $\mu$-PL condition and manipulating, this implies
		$$\EE[f(x_{t+1})-f^*] \leq (1-\kappa)[f(x_t)-f^*] + \psi.$$
		Iterating this and taking an expectation with respect to the random non-straggler set at each iteration, we have

		\begin{align*}
			\EE[f(x_{T})-f^*] &\leq \left(1-\kappa\right)^T[f(x_0)-f^*] + \psi\sum_{j=0}^{T-1}\left(1-\kappa\right)^j\\
			&\leq (1-\kappa)^T[f(x_0)-f^*] + \psi\sum_{j=0}^\infty (1-\kappa)^j\\
			&= (1-\kappa)^T[f(x_0)-f^*] + \frac{\psi}{\kappa}\\
			&= \left(1-\frac{\mu}{\beta}\right)^T[f(x_0)-f^*] + \dfrac{p\sigma^2}{2(1-p)\mu} + \dfrac{(q-p)c\sigma^2}{2(1-p)^2\mu n}.
		\end{align*}

		Here we used the fact that $0 < \kappa \leq 1$.

		Using \eqref{eq:g_beta}, and similar reasoning in the equation below it, we have

		\begin{align*}
		\EE[f(x_{t+1})] &\leq f(x_t) + \left(\dfrac{\beta\gamma^2}{2(1-p)}-\gamma\right)\|\nabla f(x_t)\|^2 + \dfrac{(q-p)\beta\gamma^2c\sigma^2}{2(1-p)^2n}.\end{align*}

		The quantity $\beta\gamma^2/2(1-p) - \gamma$ is maximized by setting $\gamma = (1-p)/\beta$. Using this step-size, we get
		\begin{align*}
		\EE[f(x_{t+1})] \leq f(x_t) - \frac{(1-p)}{2\beta}\|\nabla f(x_t)\|^2 + \dfrac{(q-p)c\sigma^2}{2\beta n}.\end{align*}
		We will again let $\kappa$ denote the condition number $\mu/\beta$ and now define $\phi$ by
		$$\phi := \dfrac{(q-p)c\sigma^2}{2\beta n}.$$
		Applying the $\mu$-PL condition and manipulating, we have
		$$\EE[f(x_{t+1})-f^*] \leq (1-(1-p)\kappa)[f(x_t)-f^*] + \phi.$$
		Iterating this and taking an expectation with respect to the random non-straggler sets, we have

		\begin{align*}
			\EE[f(x_{T})-f^*] &\leq \left(1-(1-p)\kappa\right)^T[f(x_0)-f^*] + \phi\sum_{j=0}^{T-1}\left(1-(1-p)\kappa\right)^j\\
			&\leq (1-(1-p)\kappa)^T[f(x_0)-f^*] + \phi\sum_{j=0}^\infty (1-(1-p)\kappa)^j\\
			&= (1-(1-p)\kappa)^T[f(x_0)-f^*] + \frac{\phi}{(1-p)\kappa}\\
			&= \left(1-\frac{(1-p)\mu}{\beta}\right)^T[f(x_0)-f^*] + \dfrac{(q-p)c\sigma^2}{2(1-p)\mu n}.
		\end{align*}

		Here we used the fact that $0 < (1-p)\kappa \leq 1$.
	\end{proof}	

	This result is somewhat opaque due to the presence of $p$ and $q$. We can make this more interpretable in the following theorem with only mild assumptions. We first show the following lemma about the moment $p$.

	\begin{lemma}$$p \leq \exp(-cr/n).$$\end{lemma}
	\begin{proof}
		By simple estimates,
		\begin{align*}
		\dfrac{\binom{k-\ell}{r}}{\binom{k}{r}}\leq \left(\dfrac{k-\ell}{k}\right)^r = \left(1-\dfrac{\ell}{k}\right)^r.\end{align*}
		The lemma then follows from the fact that for all $x \in \real$, $1+x \leq e^x$ and that $\ell = kc/n$.
	\end{proof}	

	If $c \geq n\ln(2)/r$, then this lemma implies that $p \leq 1/2$. Using this bound for $p$ and the fact that $q \leq 2p$ in Theorem \ref{thm:conv_rate}, we derive the following theorem.

	\begin{theorem}\label{thm:conv_rate_2}Suppose that $f$ is $\beta$-smooth and $\mu$-PL with $\beta \geq \mu > 0$ and that for all $i \in [n]$ and $x$, $\|\nabla f_i(x)\| \leq \sigma$. Further suppose that $c \geq n\ln(2)/r$.
	If $\gamma = 1/\beta$, then
	$$\Delta_T \leq \left(1-\frac{\mu}{\beta}\right)^T\Delta_0 + \dfrac{e^{-cr/n}\sigma^2}{\mu} + \dfrac{4ce^{-cr/n}\sigma^2}{\mu n}.$$
	If $\gamma = (1-p)/\beta$, then
	$$\Delta_T \leq \left(1-\frac{(1-e^{-cr/n})\mu}{\beta}\right)^T\Delta_0 +\dfrac{2ce^{-cr/n}\sigma^2}{\mu n}.$$
	\end{theorem}

	By Lemma \ref{lem:conv_pl}, vanilla gradient descent has convergence rate
	$$\Delta_T \leq \left(1-\frac{\mu}{\beta}\right)^T\Delta_0.$$
	Since using exact gradient codes computes the true gradient, gradient descent with exact gradient codes has the same convergence rate for $\mu$-PL functions. Thus, all three methods have comparable linear convergence rate up to some noise level that is $O(n^{-1}\exp(-cr/n))$. Note that if $n$ is large and $r$ is a constant fraction of $n$, this term becomes small. In the following section, we combine the convergence rates above with a probabilistic model for runtime per iteration to determine the expected overall runtime for these three methods.
	\section{Probabilistic Runtime Analysis}

	In this section, we theoretically analyze the runtime of various distributed gradient descent methods. We use the shifted exponential model of stragglers from \cite{lee2016speeding}. In this model, we assume each compute node has a runtime drawn from randomly from a shifted exponential model that depends on the fraction of tasks assigned to the node. While in the previous section we assumed we had $n$ tasks and $k$ workers, here we will assume $n = k$ to make our results easier to parse. Analogous results can be derived for the setting that $n \geq k$.

	As in \cite{lee2016speeding}, we assume that the amount of time required to run a full gradient update on a single node is a continuous nonnegative random variable $T_0$ with cumulative and density functions $Q(t)$ and $q(t)$. That is, $\PP[T_0 \leq t] = Q(t)$. When the algorithm is partitioned in to $B$ subtasks, we assume that each of the $B$ subtasks has independent runtime $T_i$ with probability density function $Q(Bt)$. This assumes that the job partitioning is symmetric and that all compute nodes have the same compute power.

	\cite{lee2016speeding} found that empirically, $Q(t)$ is often close to the cumulative distribution of a shifted exponential distribution. An exponential distribution with parameter $\lambda$ has cumulative distribution $\PP[Z \leq z] = 1-e^{-\lambda z}$. In the following, we assume the following form of $Q(t)$:
	\begin{equation}\label{eq:h_distr}
	\PP[T_0 \leq t] = Q(t) = 1-e^{-\lambda(t-1)}.\end{equation}
	Here $\lambda$ is the {\it straggling parameter}. If $\lambda$ is smaller, the straggler effect is more pronounced. Elementary analysis shows that given a partition in to $B$ tasks, we have
	\begin{equation}\label{eq:prob_comp}
	\PP[T_i \leq B^{-1} + \tau] = e^{-B\lambda\tau}.\end{equation}
	Thus, if we partition in to $B$ groupings of tasks, then each compute node has runtime equal to $B^{-1} + Y$ where $Y$ is drawn from an exponential distribution with parameter $B\lambda$. As the number of groupings increases, the expected runtime for each compute node decreases.

	We will compare the expected runtime of uncoded gradient descent, coded gradient descent, and \erasegd{}. These correspond to using no gradient code, an exact gradient codes (EGC), and an approximate gradient codes (AGC). While uncoded and EGC have the same convergence rate, the runtime needed to find compute a full gradient update may be different. AGC differs in both computation time and in the number of iterations required to reach a given accuracy.

	Denote the runtime of the $i$-th compute node for these algorithms by $T^{\unc}_i, T^{\egc}_i, T^{\agc}_i$, and let $T^{\unc}_{\tot}, T^{\egc}_{\tot}, T^{\agc}_{\tot}$ denote the overall runtimes of these algorithms. Since we assume $n = k$, in vanilla gradient descent we assign one gradient to each worker. Therefore, the runtime of a single compute node is distributed according to $F(nt)$. The expected total runtime of vanilla gradient descent is therefore given by
	\begin{equation}\label{eq:runtime_vanilla_0}
	\EE\left[T^{\unc}_{\tot}\right] = \EE\left[\max_{i \in [n]}T^{\unc}_i\right].\end{equation}

	For exact coded gradient descent, suppose we assign $c$ tasks per worker. Thus, we partition the $n$ gradients in to $B = n/c$ groups. Hence, $\PP[T^{\egc}_i \leq t] = Q(Bt)$. Exact coded gradient descent updates once we have constructed the entire gradient. Therefore, we need one compute node for each of the $B$ groupings to finish. Define
	$$Y^{\egc}_i = \min_{j\in [c]}T^{\egc}_{(i-1)+\frac{n}{c}j}.$$
	In other words, $Y^{\egc}_i$ is the time required for the $i$-th grouping of tasks to be completed. Therefore,
	\begin{equation}\label{eq:runtime_gc_0}
	\EE\left[T^{\egc}_{\tot}\right] = \EE\left[\displaystyle\max_{i \in [n/c]}Y^{\egc}_i\right].\end{equation}

	Finally, suppose we perform approximate coded gradient descent where we wait for at most $r$ non-straggler nodes to finish computing and we assign $c$ tasks per worker. Note that this corresponds to partitioning in to $B = n/c$ task groups, so $\PP[T^{\agc}_i \leq t] = Q(Bt)$. The runtime is then upper bounded by the expected runtime of the $r$-th node. Let $T^{\agc}_{(r)}$ denote the $r$-th smallest value of the $T_{i}^{\agc}$. This is known as the $r$-th order statistic. Then we have,
	\begin{equation}\label{eq:runtime_agc_0}
	\EE\left[T^{\agc}_{\tot}\right] \leq \EE\left[T^{\agc}_{(r)}\right].\end{equation}

	Note that this inequality is generally not tight. This is due to the fact that we may compute $\nabla f(x)$ with fewer than $r$ compute nodes. However, when $r$ is small enough, bound will be almost tight. We will analyze $T^{\agc}_{(r)}$ since it will be independent from the approximation error of the gradient update.

	We now use the following two well-known results about the maximum of exponential random variables.
	\begin{lemma}\label{lem:order_stat}
	Suppose that $Z_1,\ldots, Z_n$ are independent exponential random variables with parameter $\lambda$. Let $Z_{(p)}$ denote the $p$-th order statistic (the $p$-th smallest value) of the $Z_i$. Then
	$$\EE\left[Z_{(p)}\right] = \frac{H_{n}-H_{n-p}}{\lambda}$$
	where $H_m := \sum_{i=1}^m i^{-1}$ for $m \in \mathbb{Z}_{> 0}$ and $H_m = 0$ for $m = 0$.\end{lemma}

	\begin{lemma}\label{lem:min}
	Suppose that $Z_1,\ldots, Z_c$ are independent exponential random variables with parameter $\lambda$. Then the random variable $Z = \min_{i \in [c]}Z_i$ is distributed as an exponential random variable with parameter $c\lambda$.\end{lemma}

	By standard estimates, $H_m = \Theta(\log m)$ for $m \geq 1$. Combining \eqref{eq:prob_comp}-\eqref{eq:runtime_agc_0} and Lemmas \ref{lem:order_stat} and \ref{lem:min}, we derive the following lemma.

	\begin{lemma}\label{lem:runtimes}
		The expected runtimes $T^{\unc}_{\tot}, T^{\egc}_{\tot}, T^{\agc}_{\tot}$ satisfy the following:
		\begin{equation}\label{eq:runtime_vanilla}
		\EE\left[T^{\unc}_{\tot}\right] = \frac{1}{n}+\frac{H_n}{\lambda n}.\end{equation}
		\begin{equation}\label{eq:runtime_gc}
		\EE\left[T^{\egc}_{\tot}\right] = \frac{c}{n}+\frac{H_n-H_c}{\lambda n}.\end{equation}
		\begin{equation}\label{eq:runtime_agc}
		\EE\left[T^{\agc}_{\tot}\right] \leq \EE\left[T^{\agc}_{(r)}\right] = \frac{c}{n}+\frac{c}{\lambda n}(H_n-H_{n-r}).\end{equation}
	\end{lemma}

	Using this lemma, we can now analyze the expected amount of time each algorithm requires to achieve a given accuracy $\epsilon$. We now wish to determine for which $\lambda$ gradient coding outperforms vanilla gradient descent. One can show that the right-hand side of \eqref{eq:runtime_gc} is minimized when $c = 1$ if $\lambda \geq 1$, and $c = \frac{1}{\lambda}$ for $0 < \lambda < 1$. Thus, we will assume that $\lambda = \frac{1}{c}$ for $c \in \mathbb{Z}_{> 0}$. Note that $\lambda < 1$ corresponds to having a larger tail in the exponential distribution, which intuitively signifies that coded distributed methods have more utility.

	We now wish to show that AGC gradient descent has smaller expected runtime than uncoded gradient descent or EGC gradient descent up to certain noise floors. Define
	$$\epsilon_0 := \dfrac{3ce^{-cr/n}\sigma^2}{\mu n}.$$

	Suppose we initialize at $x_0$ and let $\Delta_0 = f(x_0)-f^*$. Assume $f$ is $\mu$-PL and $\beta$-smooth, and let $\kappa = \mu/\beta$. In our AGC, suppose we wait for $r = \delta n$ non-stragglers for $\delta \in (0,1)$. Let $\lambda = 1/c$ for some positive integer $c$, and for exact and approximate gradient codes assign $c$ tasks per worker. Let $\eta = 1-e^{-cr/n} = 1-e^{-c\delta}$. For reasonable values of $c$, $\eta$ is close to $1$. We get the following theorem.





	\begin{theorem}\label{thm:tot_time}
	Let $T^{\unc}_{\epsilon}, T^{\egc}_{\epsilon}, T^{\agc}_{\epsilon}$ denote the amount of time required for uncoded and EGC gradient descent with step-size $\gamma = 1/\beta$, and AGC gradient descent with step-size $\gamma = (1-p)/\beta$ to reach error $\epsilon \geq \epsilon_0$. Then under the above probabilistic delay model:
	$$\EE[T^{\unc}_{\epsilon}] \leq \dfrac{\log(\Delta_0/\epsilon)}{\log(1/(1-\kappa))}\dfrac{c\log(n)+c+1}{n}.$$
	$$\EE[T^{\egc}_{\epsilon}] \leq \dfrac{\log(\Delta_0/\epsilon)}{\log(1/(1-\kappa))}\dfrac{c\log(n/c)+c+1}{n}.$$
	$$\EE[T^{\agc}_{\epsilon_2}] \leq \dfrac{\log(3\Delta_0/\epsilon)}{\log(1/(1-\eta\kappa))}\dfrac{c^2\log(1/(1-\delta))+c^2+c}{n}.$$\end{theorem}

	\begin{proof}
		Throughout this proof, we will rely on the fact that for $n \geq 1$, we have
		\begin{equation}\label{eq:harmonic}
		\log(n) \leq H_n \leq \log(n)+1.\end{equation}

		First, consider uncoded gradient descent. Fix some initial $x_0$ and let $\Delta_0 = f(x_0)-f^*$. By Lemma \ref{lem:conv_pl}, the number of iterations $N_\epsilon$ required to reach accuracy $\epsilon > 0$ is
		$$N \leq \dfrac{\log(\Delta_0/\epsilon)}{\log(1/(1-\kappa))}.$$

		By Lemma \ref{lem:runtimes}, the expected runtime per iteration is $1/n + H_n/\lambda n$. By assumption, $\lambda = 1/c$. Using this and \eqref{eq:harmonic} leads to the first and second part of the theorem. Here we utilize the fact that the number of iterations for uncoded gradient descent and EGC to to reach a given accuracy from an initial iterate is deterministic, and so the expected runtime per iteration is independent of the number of iterations.

		For the third part, we must make some minor adjustments. First, the expected number of iterations required to reach a given accuracy is a bit different. Suppose $N$ is large enough such that
		$$\left(1-\eta\kappa\right)^N\Delta_0 \leq \dfrac{\epsilon}{3}.$$
		
		Then by Theorem \ref{thm:conv_rate_2}, we have
		$$\Delta_N \leq \dfrac{\epsilon}{3} + \dfrac{2e^{-cr/n}\sigma^2}{\mu n} \leq \epsilon.$$
		Here we used the fact that $\epsilon_0 \leq \epsilon$. Simple analysis shows that this holds for all $N$ such that
		\begin{equation}\label{eq:aux1}
		N \geq \dfrac{\log(3\Delta_0/\epsilon)}{\log(1/(1-\eta\kappa))}.\end{equation}

		By Lemma \ref{lem:runtimes}, the expected runtime per iteration is bounded above by $c/n + (c/\lambda n)(H_n-H_{n-r})$. Using $\lambda = 1/c$ and \eqref{eq:harmonic}, we find that this is upper bounded by $(c^2\log(n/(n-r)) + c^2 + c)/n$. Since $r = \delta n$, we have
		$$\dfrac{c}{n} + \dfrac{c}{\lambda n}(H_n-H_{n-r}) \leq \dfrac{c^2\log(1/(1-\delta)) + c^2 + c}{n}.$$

		Second, note that the number of iterations $N_\epsilon$ required to reach a given accuracy $\epsilon$ is not independent from the runtime $T^{\agc}_{\tot}$. per iteration of AGC gradient descent. However, $N_\epsilon$ is independent from the $r$-th moment of the $T^{\agc}_i$, as the $r$-th moment does not depend on which of the nodes are stragglers, since they are all independent. However, our bound from Lemma \ref{lem:runtimes} applies to $T^{\agc}_{(r)}$, the $r$-th moment. Therefore, the total runtime required to reach a given accuracy $\epsilon$ satisfies
		$$\EE[T^{\agc}_\epsilon] \leq \EE[N_\epsilon T^{\agc}_{(r)}] = \EE[N_\epsilon]\EE[T^{\agc}_{(r)}].$$
		Applying Lemma \ref{lem:runtimes}, we complete the proof.
	\end{proof}	

	Suppose that $c$ is sufficiently large and $\delta$ is some fixed constant, so that $e^{-c\delta}$ is close to 0. Ignoring all constants except $n$, $c$, and $\delta$, this essentially states that
	$$\EE[T^{\unc}_{\epsilon}] = O\left(\dfrac{c\log(n)}{n}\right).$$
	$$\EE[T^{\egc}_{\epsilon}] = O\left(\dfrac{c\log(n/c)}{n}\right).$$
	$$\EE[T^{\agc}_{\epsilon}] = O\left(\dfrac{c^2\log(1/(1-\delta))}{n}\right).$$
	Therefore, AGC gradient descent can lead to almost a $\log(n)$ speedup over uncoded gradient descent and EGC gradient descent, even up to error levels as small as $e^{-c}/n$. In particular, as $n\to \infty$, the speedup gain of AGC increases while the error level tends to 0.
	\section{Experiments}
	\label{sec:experiments}
	In this section we present an experimental study of \erasegd{} on distributed clusters. We compare \erasegd{} method to i) uncoded distributed gradient descent where data is partitioned among all workers with no replication and ii) exact gradient coded distributed gradient descent from \cite{tandon2017gradient}. For brevity, we will often refer to the uncoded method as uncoded GD, the exact coded method as EGC, and \erasegd{} as AGC. Both AGC and EGC use the fractional repetition code above with $c$ tasks assigned per worker. While \cite{tandon2017gradient} also presents a cyclic repetition code, they show that it is generally slower than the fractional repetition code in practice.

	Uncoded and EGC gradient descent both wait until enough compute nodes so that the exact gradient can be computed. However, \erasegd{} will stop earlier if some maximum fraction $\delta$ of the compute nodes have finished. For each experimental setup, we tune $\delta$ to get the best end-to-end performance. Our results generally show that even in simple distributed settings, \erasegd{} leads to speedups in training time over both other methods. Moreover, as the straggler effect gets more pronounced, \erasegd{} leads to even larger speedups.

\subsection{Experimental Setup}
We implemented all algorithms in python using MPI4py \cite{dalcin2011parallel}. We compared uncoded GD, exact coded GD, and \erasegd{} with different redundancies across a distributed cluster consists of a parameter server (PS) node and 30 worker nodes. The worker nodes are \texttt{m1.small} instances on Amazon EC2, which are relatively small and low-cost. We used a larger instance \texttt{c3.8xlarge} as our PS node to mitigate additional inherent overhead over the cluster. In our experiments, each worker is initially assigned a number of partitions of the data, with the number depending on the method used. Then, in the $t$-th iteration, the PS broadcasts the latest model $x^{(t)}$ to all workers. Each worker computes the gradient(s) of this model with respect to their data partition. Each worker then sends their gradient(s) to the PS. For uncoded GD and EGC, once enough workers have finished such that we can compute the full gradient, we update the model according to the gradient step and move to the next iteration. In \erasegd{}, once a fraction $\delta$ of the workers have finished, we move on to the next iteration, regardless of whether we have computed the full gradient.

We tested on both train logistic regression and least squares tasks. While convex, these problems often have hundreds of thousands of features in practical applications. Moreover, such models are frequently used in practice and trained with distributed systems. We also performed versions of our experiments with extra artificial delays (using \texttt{time.sleep}) in the nodes to simulate practical scenarios where the communication overheads are heavy. The length (in seconds) of the delay for each worker was drawn independently from an exponential distribution with parameter $\lambda = 1/2$. 
	
\subsection{Datasets and Models}
The details of datasets and models used in our experiments are listed in Table \ref{Tab:DataStat}. We used the Amazon Employee Access (Amazon)\footnote{ \url{https://www.kaggle.com/c/amazon-employee-access-challenge}} and Forest Covertype (Covertype)\footnote{ \url{http://archive.ics.uci.edu/ml/datasets/Covertype}} datasets for logistic regression tasks and the House Sales in King County (KC Housing) \footnote{ \url{https://www.kaggle.com/harlfoxem/housesalesprediction}} dataset for least squares. For classification tasks, we used test set AUC \cite{auc} as the metric of model performance while mean square error for regression tasks. For all datasets, we deployed one-hot encoding on each particular feature and report the preprocessed dimensions in Table \ref{Tab:DataStat}. Note that while the initial Forest Covertype dataset has seven labels, for the purposes of logistic regression we use only use the $m = 495141$ samples with the two most common labels i.e. \textit{Spruce-Fir} and \textit{Lodgepole Pine}.

\begin{table}[htp]
		\centering
		\caption{The datasets used, their associated learning models and corresponding parameters.}
		{\small
		\begin{tabular}{|c|c|c|c|}
			\hline Dataset
			& Amazon & Covertype  & KC Housing \bigstrut\\
			\hline
			\# data points & 26,215 & 495,141 & 17,290 \bigstrut\\
			\hline
			Model & logistic & logistic & least squares \bigstrut\\
			\hline
			Dimension & 241,915 & 15,509 & 27,654 \bigstrut\\
			\hline
			Learning Rate & $\gamma_t=10$ & $\gamma_t=0.1$ & $\gamma_t=0.1\times(0.99)^t$ \bigstrut\\
			\hline
			\end{tabular}%
		}
		\label{Tab:DataStat}%
\end{table}
\begin{figure*}[htp]
	\vspace{-0.2cm}
	\centering
	\subfigure[Amazon]{\includegraphics[width=0.3\linewidth]{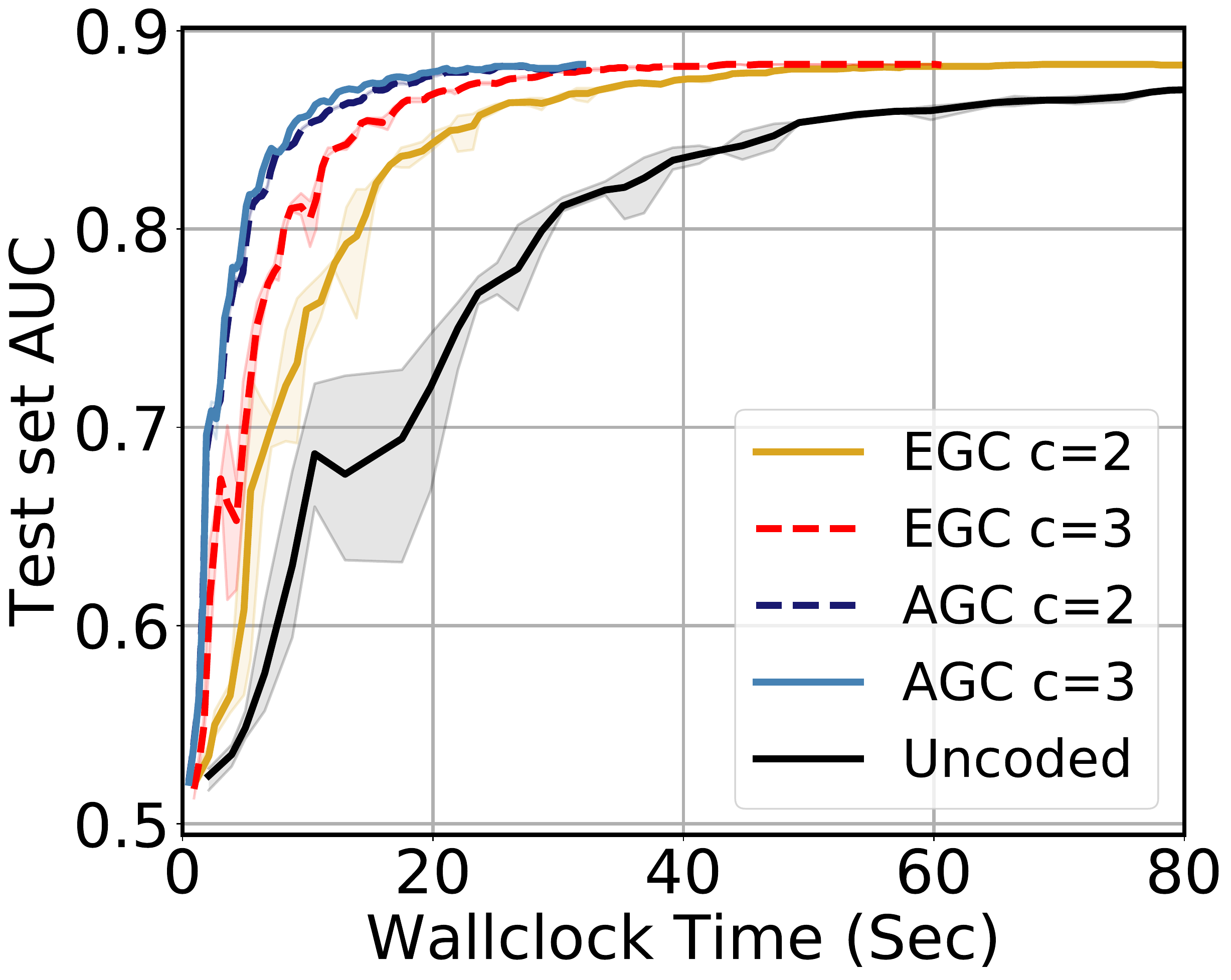}\label{fig:main_results_a}}
	\subfigure[Covertype]{\includegraphics[width=0.3\linewidth]{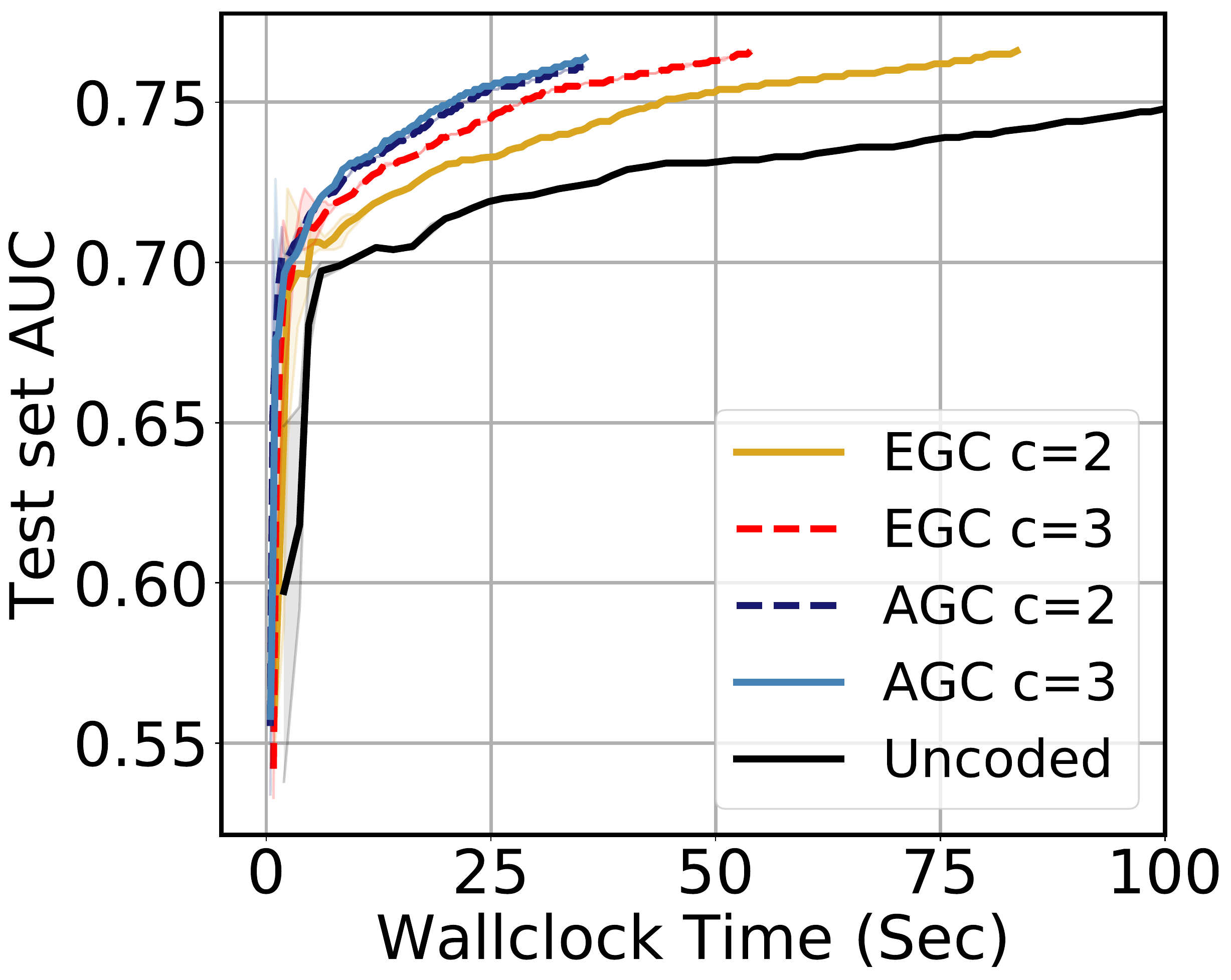}\label{fig:main_results_b}}
	\subfigure[KC Housing]{\includegraphics[width=0.3\linewidth]{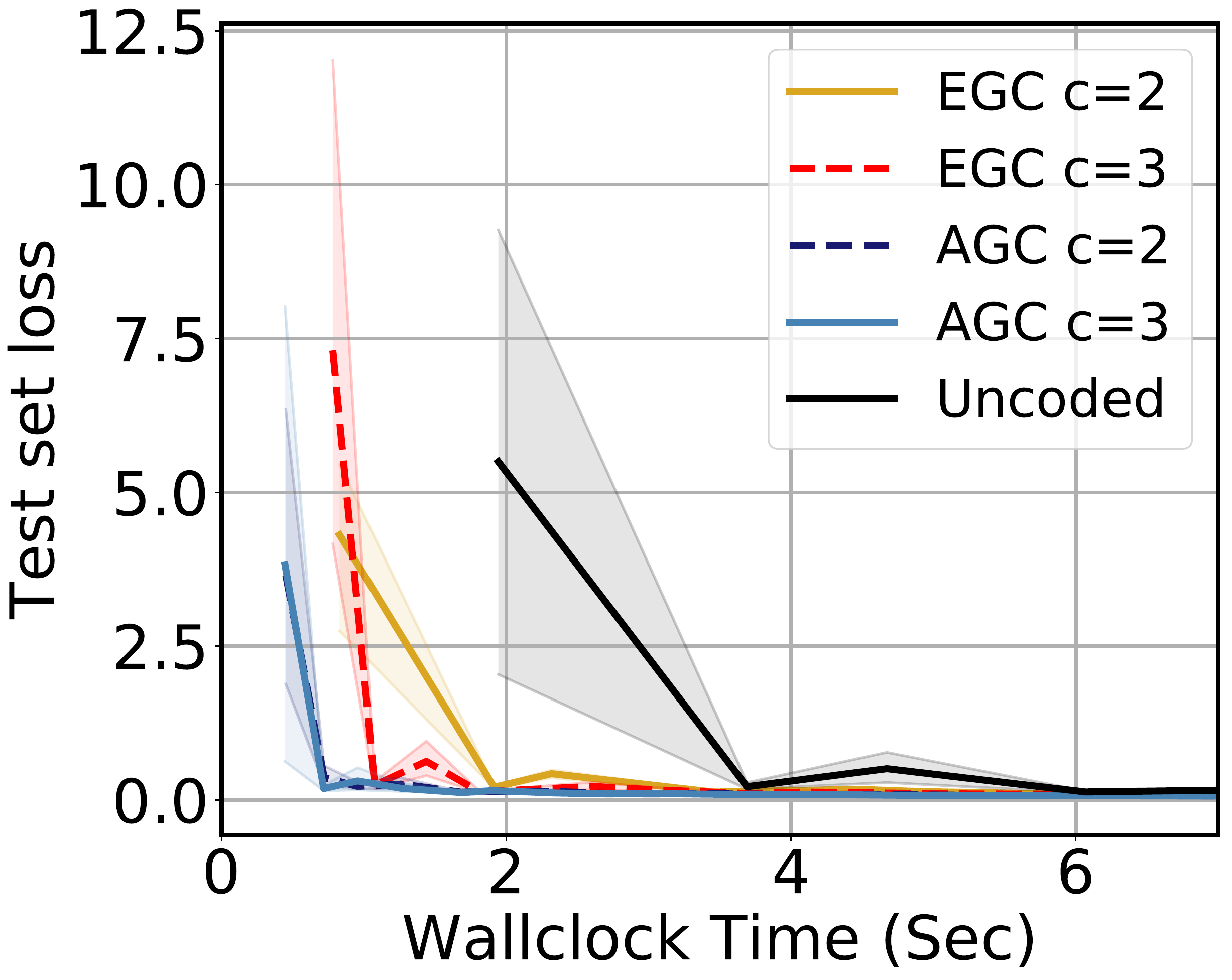}\label{fig:main_results_c}}
	\\
	\vspace{-0.4cm}
	\subfigure[Amazon]{\includegraphics[width=0.3\linewidth]{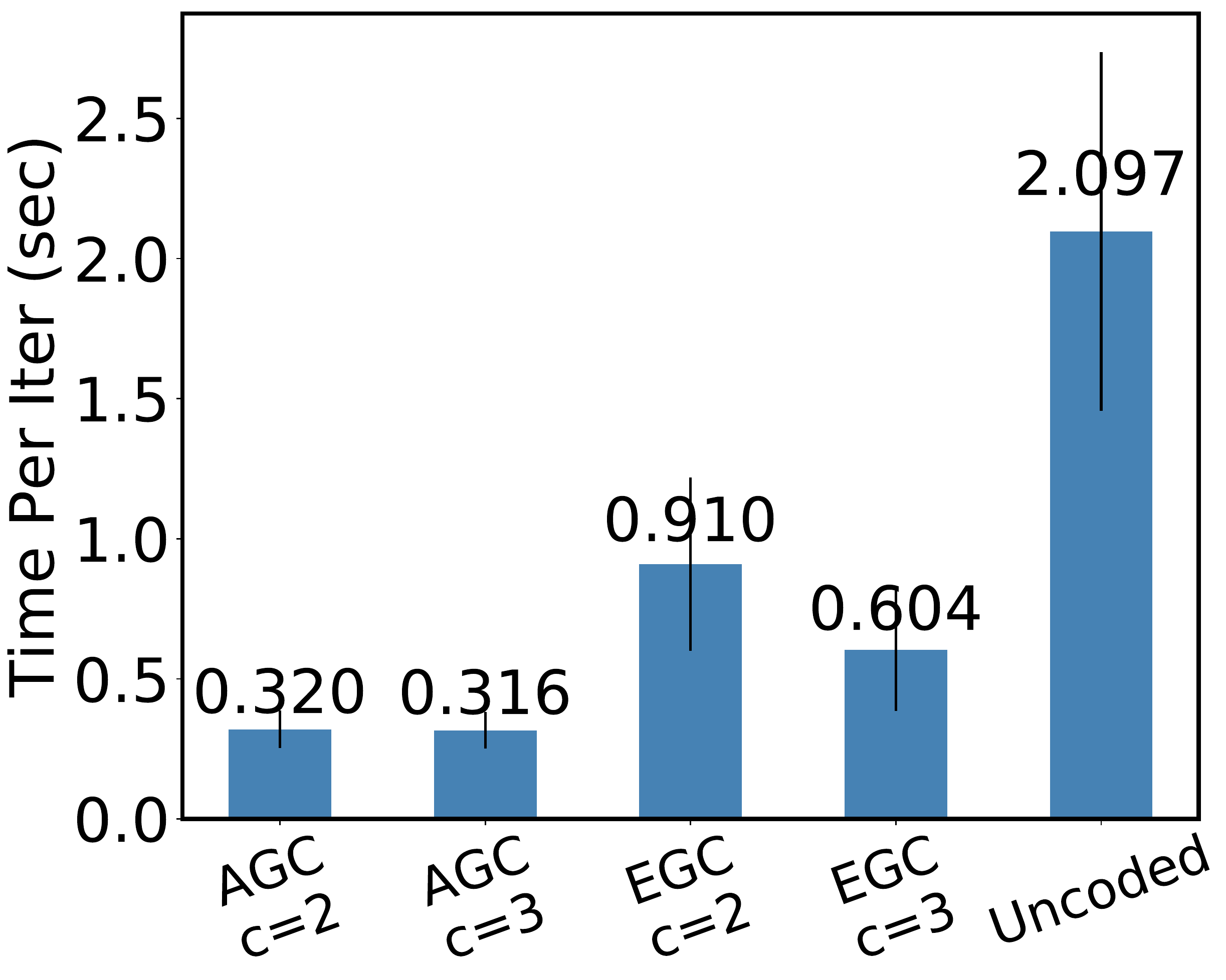}\label{fig:main_results_d}}
	\subfigure[Covertype]{\includegraphics[width=0.3\linewidth]{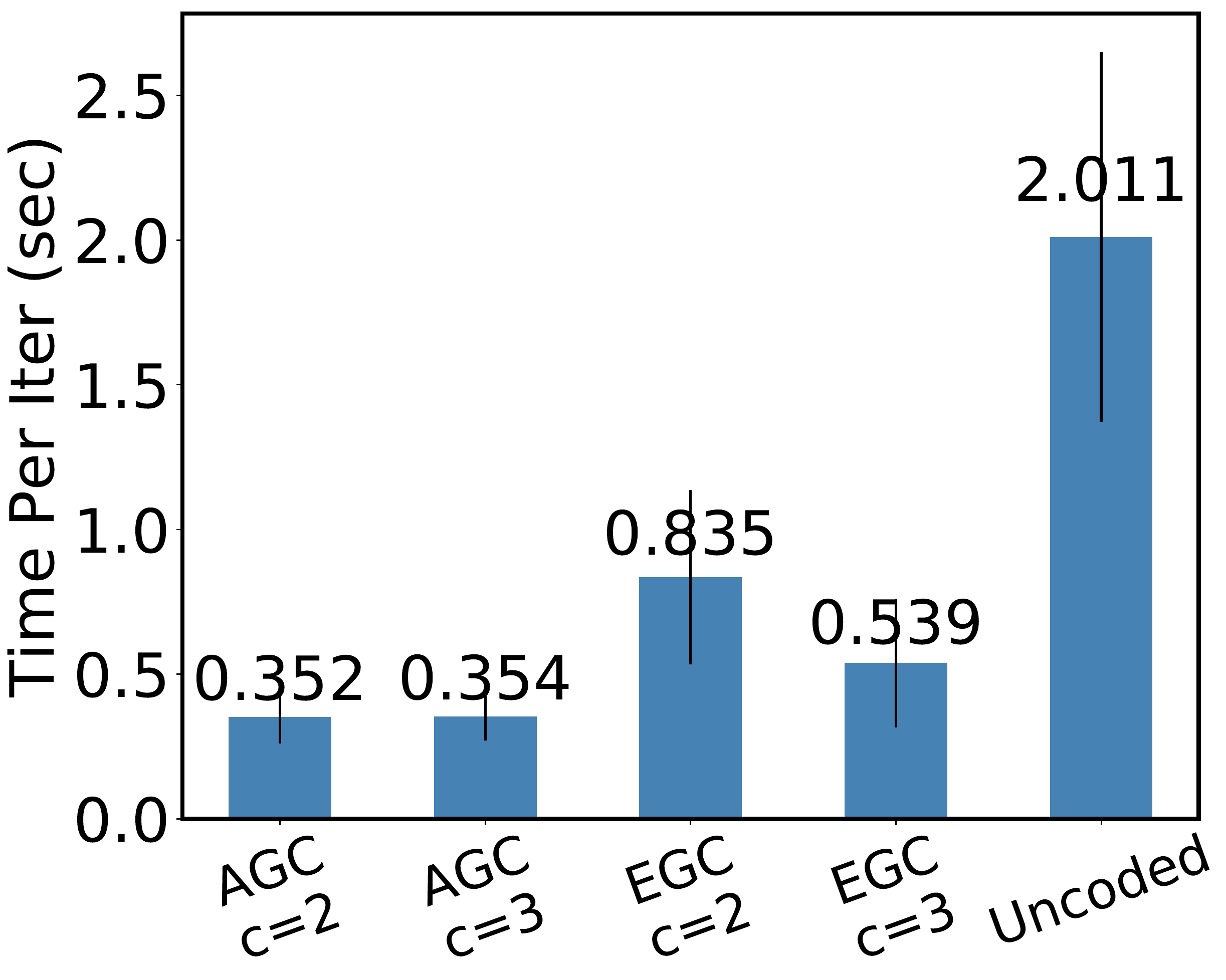}\label{fig:main_results_e}}
	\subfigure[KC Housing]{\includegraphics[width=0.3\linewidth]{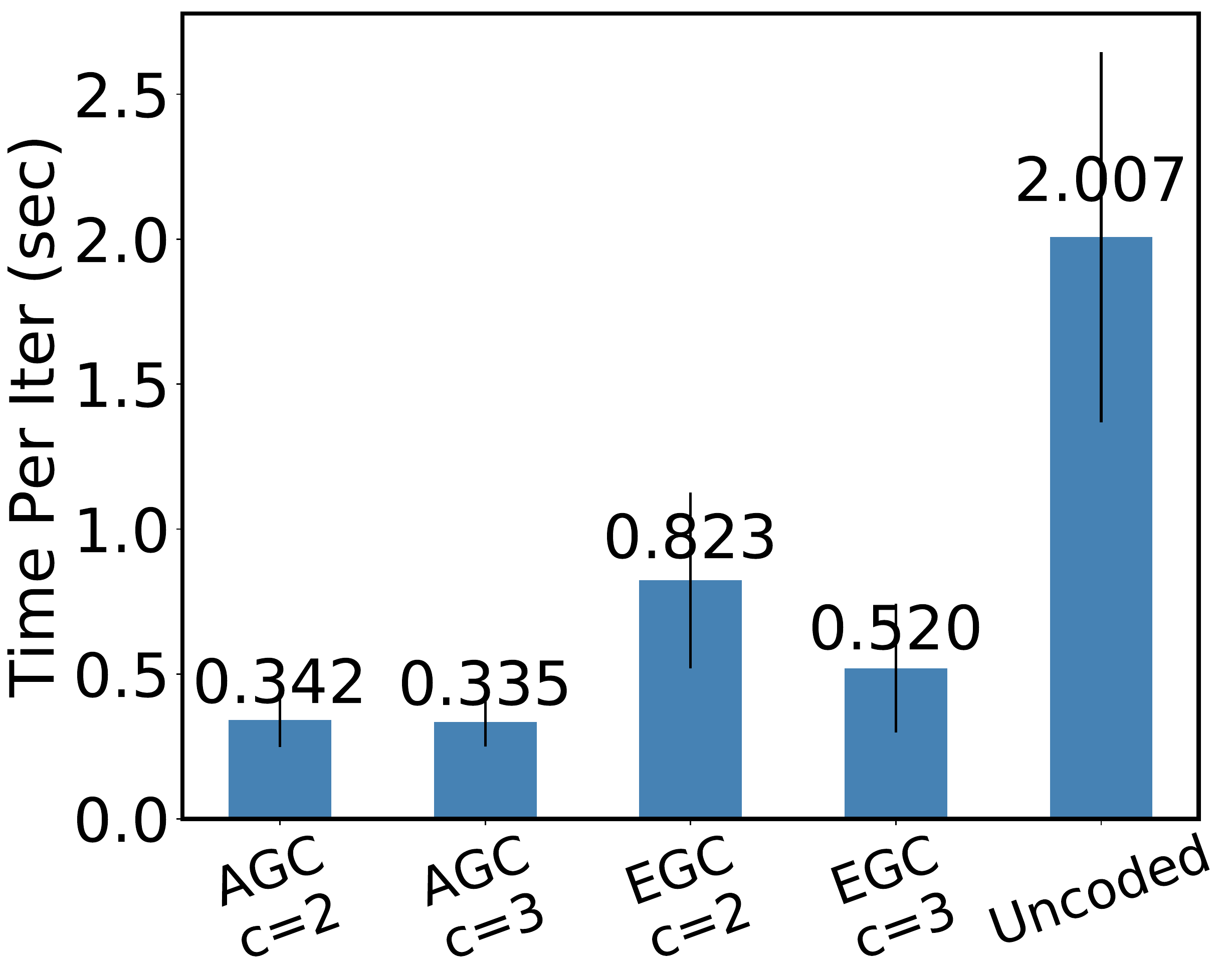}\label{fig:main_results_f}}	
	\vspace{-0.3cm}
	\caption{Results on real world datasets with artificial simulated stragglers: (a) convergence performance on Amazon dataset, (b) convergence performance on Covertype dataset, (c) convergence performance on KC Housing dataset, (d) per iteration runtime on the Amazon dataset, (e) per iteration runtime on the Covertype dataset, (f) per iteration runtime on the KC Housing dataset.}
	\label{fig:main_results}
\end{figure*}
\begin{figure*}[htp]
	\vspace{-0.2cm}
	\centering
	\subfigure[Amazon dataset]{\includegraphics[width=0.31\linewidth]{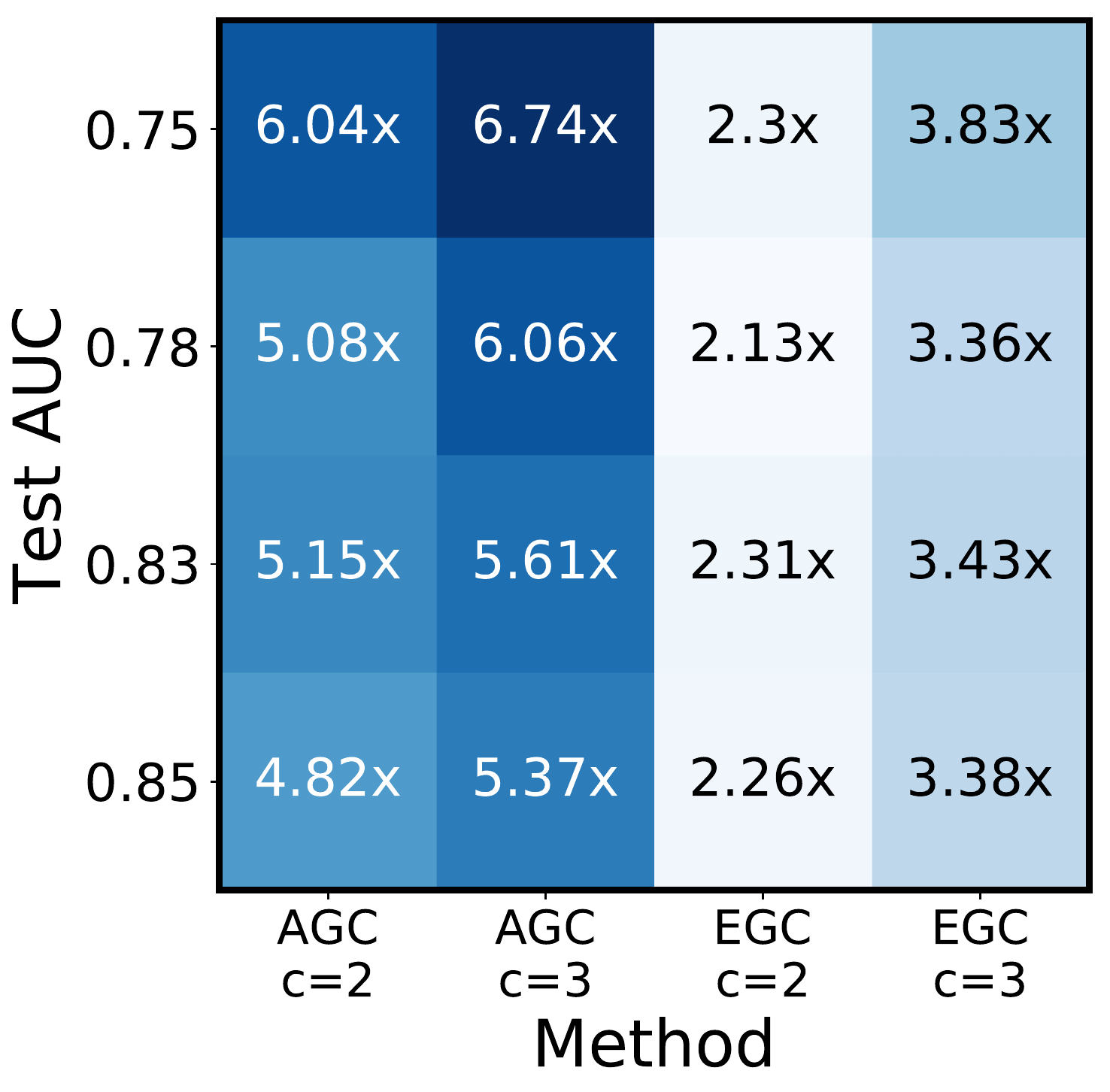}\label{fig:speedups_amazon}}
	\subfigure[Covertype dataset]{\includegraphics[width=0.31\linewidth]{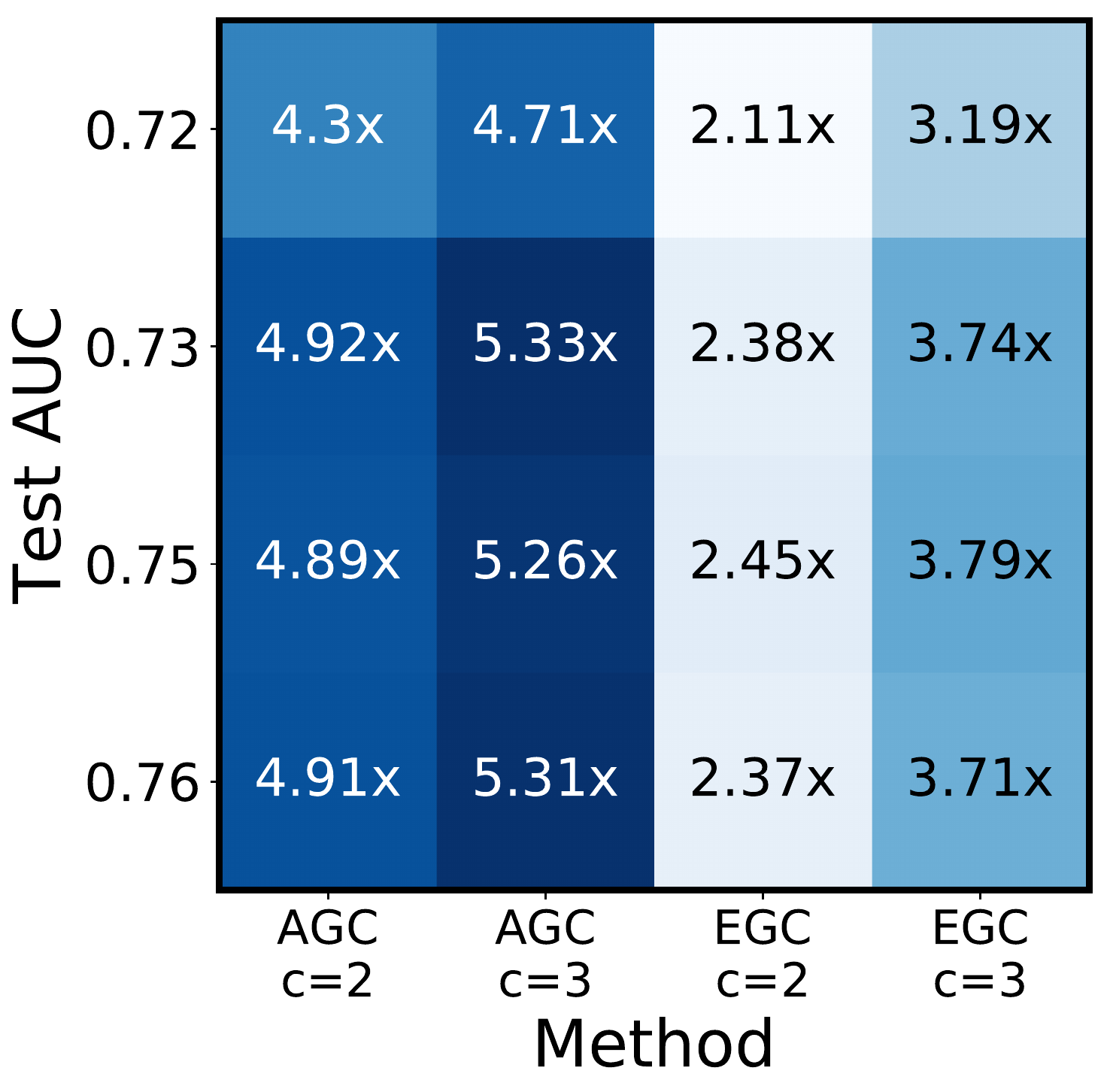}\label{fig:speedups_covertype}}
	\subfigure[KC Housing dataset]{\includegraphics[width=0.31\linewidth]{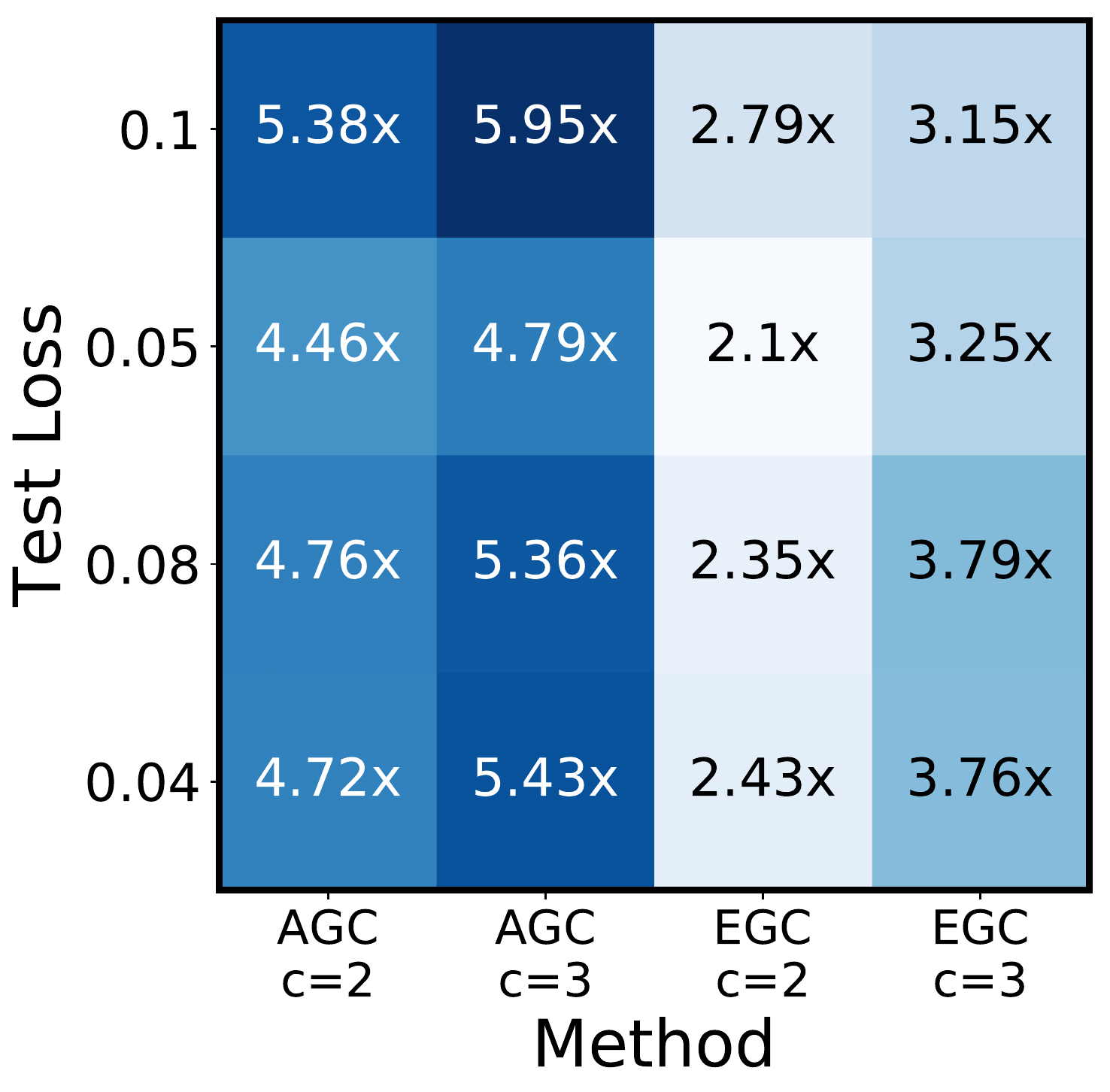}\label{fig:speedups_kc_house}}
	\vspace{-0.3cm}
	\caption{Speedups of AGC and EGC over uncoded GD under simulated stragglers on (a) the Amazon dataset, (b) the Covertype dataset, (c) KC Housing dataset}
	\label{fig:speedups}
	\vspace{-5mm}
\end{figure*}
	


\subsection{Results}

\paragraph{Amazon dataset.}  The average times per iteration on the Amazon dataset are given in Figure \ref{fig:main_results_d}, while the time to reach a given test AUC is given in Figure \ref{fig:main_results_a}. Here, we assign $c = 2$ or $c = 3$ tasks per worker in \erasegd{}. For both choice of $c$, we wait for at most $11/30 \approx 36.7\%$ of the workers to finish. Note that intuitively, as $c$ increases we want to decrease $\delta$, as otherwise \erasegd{} will, with high probability, compute the full gradient and therefore be indistinguishable from EGC. Our results show that both \erasegd{} and EGC outperform uncoded GD. Moreover, as reflected in our theory, \erasegd{} seems to consistently converge faster than EGC. Speedups for both \erasegd{} and EGC were measured under simulated straggler effect. As shown in Figure \ref{fig:speedups_amazon}, we observed that both \erasegd{} and EGC perform significantly faster than uncoded GD. Moreover, for a fixed redundancy ratio $c$, \erasegd{} attains up to 3 times speedup gain over EGC.



\paragraph{Covertype dataset.} We repeated the aforementioned experimental process. The results are plotted in Figures \ref{fig:main_results_b} and \ref{fig:main_results_e}. We observed that \erasegd{} and EGC have nearly the same performance on this dataset, both \erasegd{} and EGC successfully mitigate the straggler effects. Speedups showed in Figure \ref{fig:speedups_covertype} provide further evidence of the benefit of \erasegd{}, as we consistently observe speedup gains around $2\times$ over EGCs.

\paragraph{KC Housing dataset.}	We extended our experimental study to least squares regression on the KC Housing dataset using the same process. The results are showed in in Figures \ref{fig:main_results_c} and \ref{fig:main_results_f}. We observed that \erasegd{} attains a much smaller per iteration runtime and a substantial end-to-end speedup over EGC and uncoded GD. Figure \ref{fig:speedups_kc_house} indicates that \erasegd{} has better straggler mitigation performance and leads to up to $3\times$ faster end-to-end speedups.

\begin{figure*}[htp]
	\centering
	\subfigure[Amazon]{\includegraphics[width=0.31\linewidth]{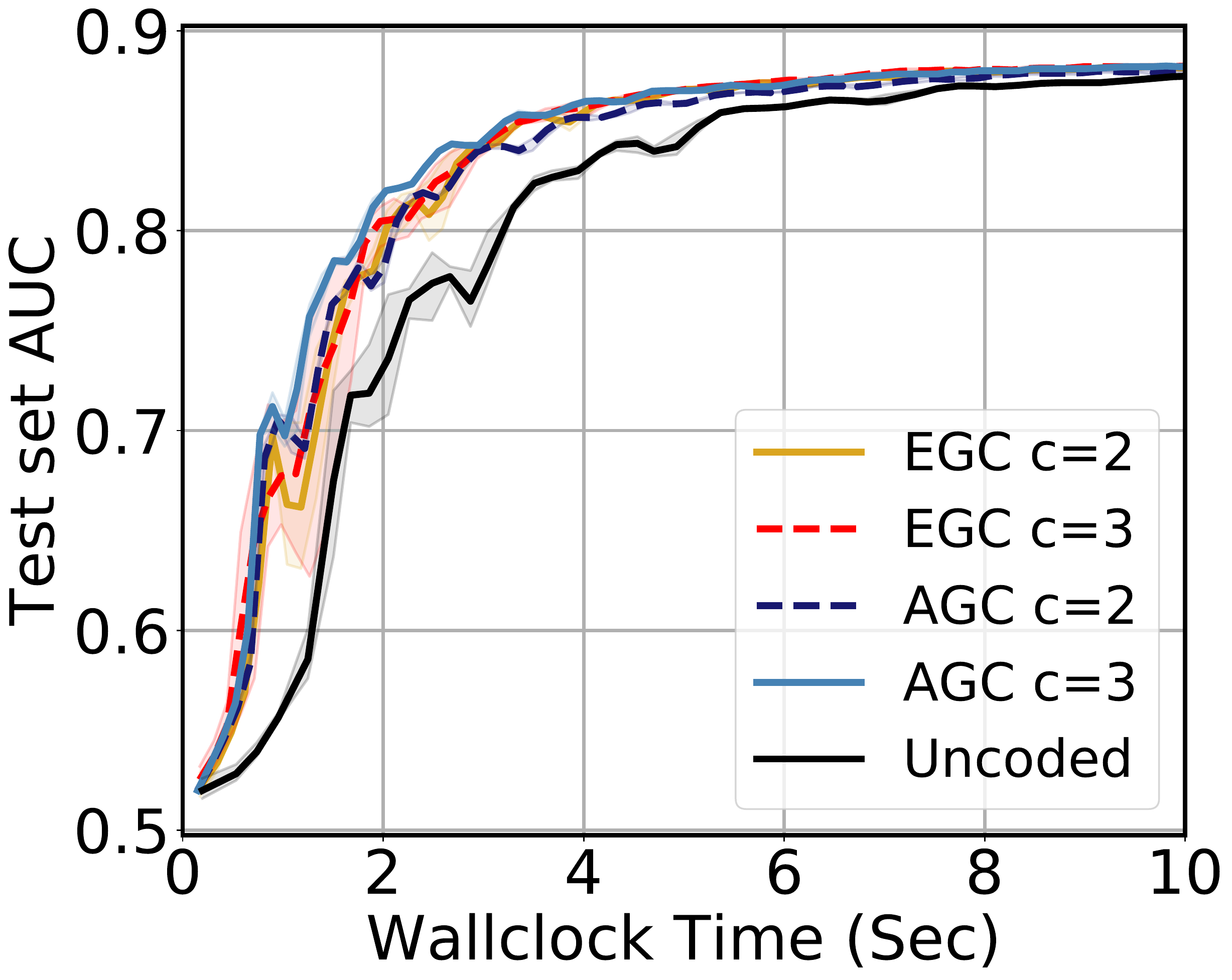}\label{fig:additional_results_a}}
	\subfigure[Covertype]{\includegraphics[width=0.31\linewidth]{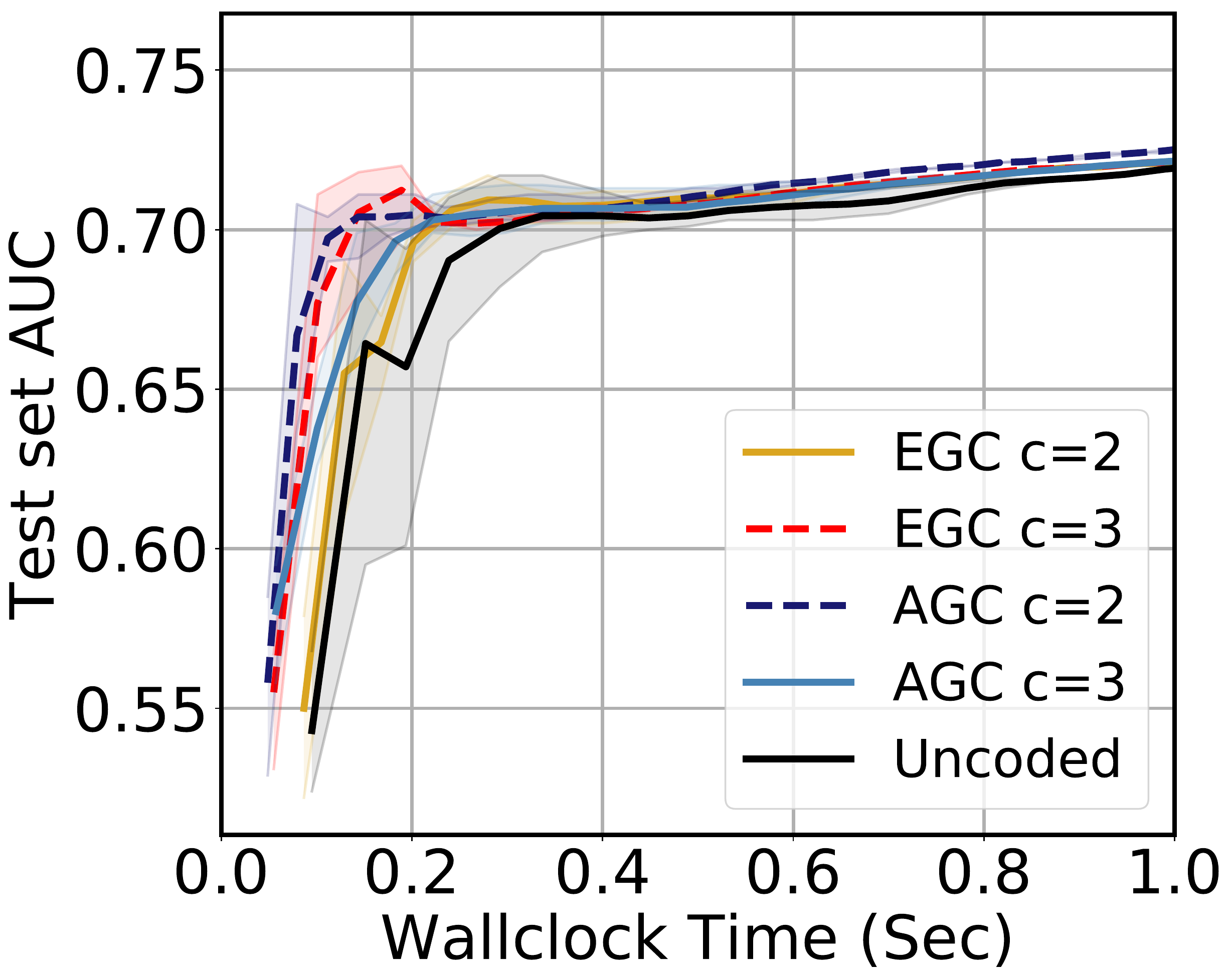}\label{fig:additional_results_b}}
	\subfigure[KC Housing]{\includegraphics[width=0.31\linewidth]{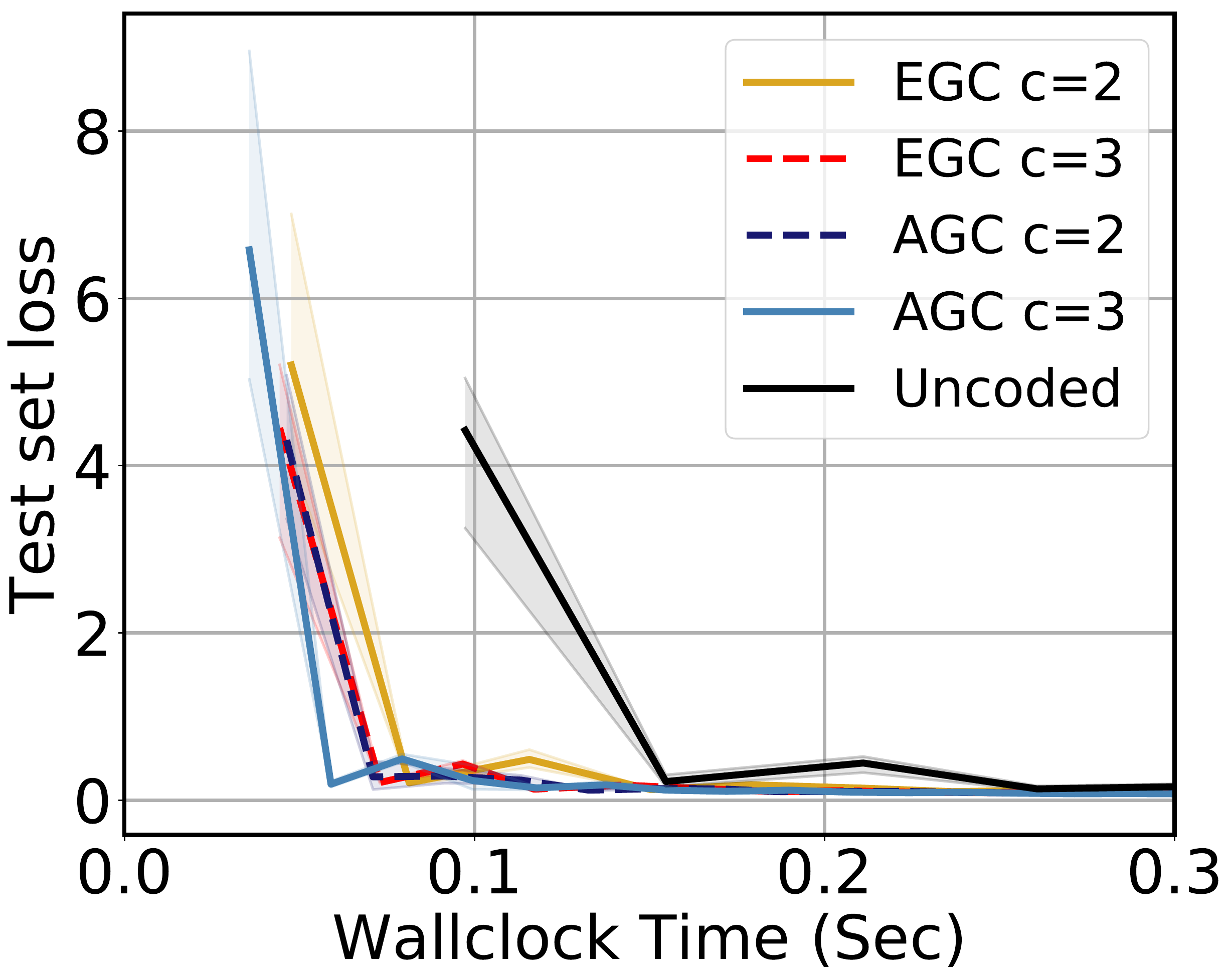}\label{fig:additional_results_c}}
	\\
	\vspace{-0.4cm}
	\subfigure[Amazon]{\includegraphics[width=0.31\linewidth]{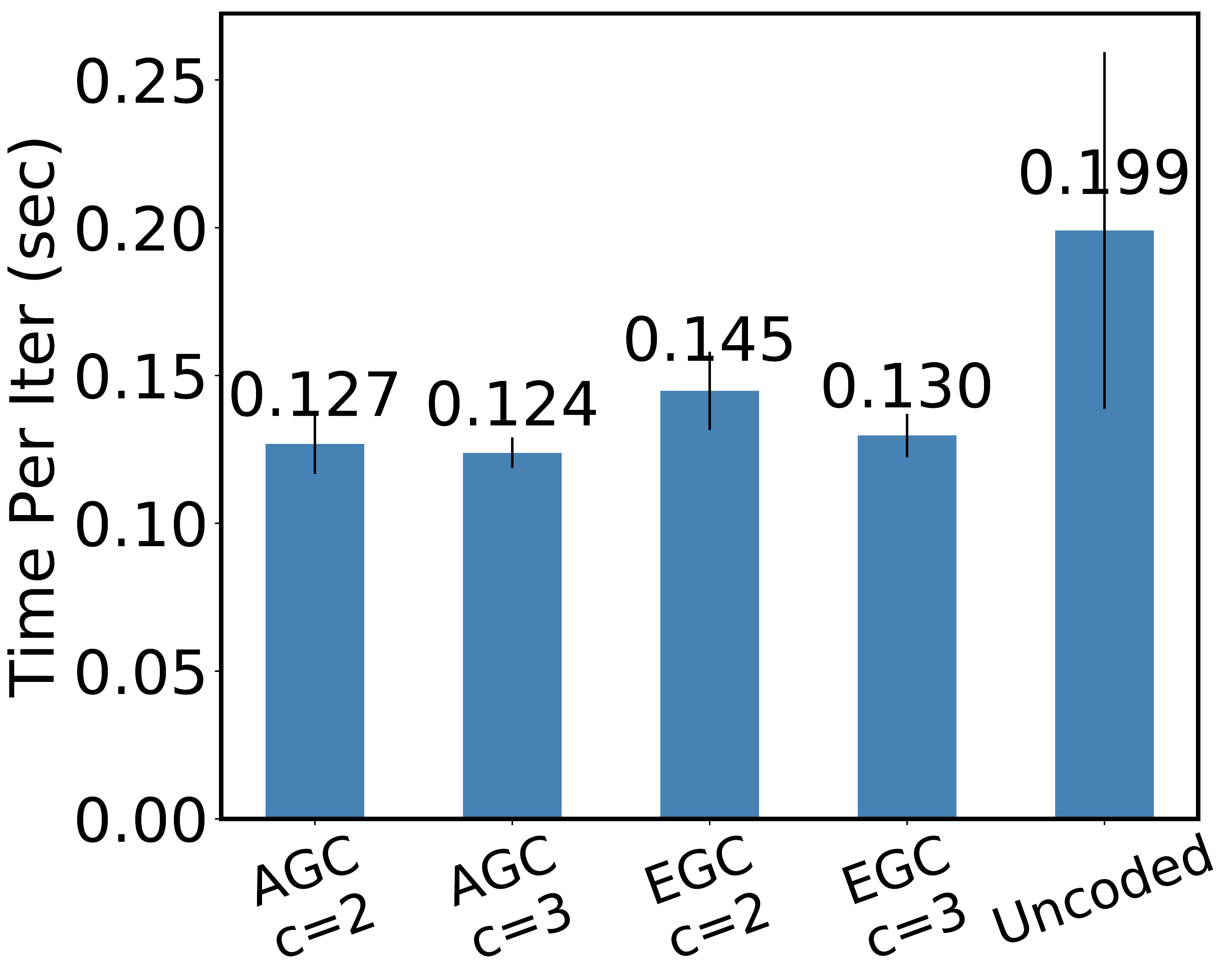}\label{fig:additional_results_d}}
	\subfigure[Covertype]{\includegraphics[width=0.31\linewidth]{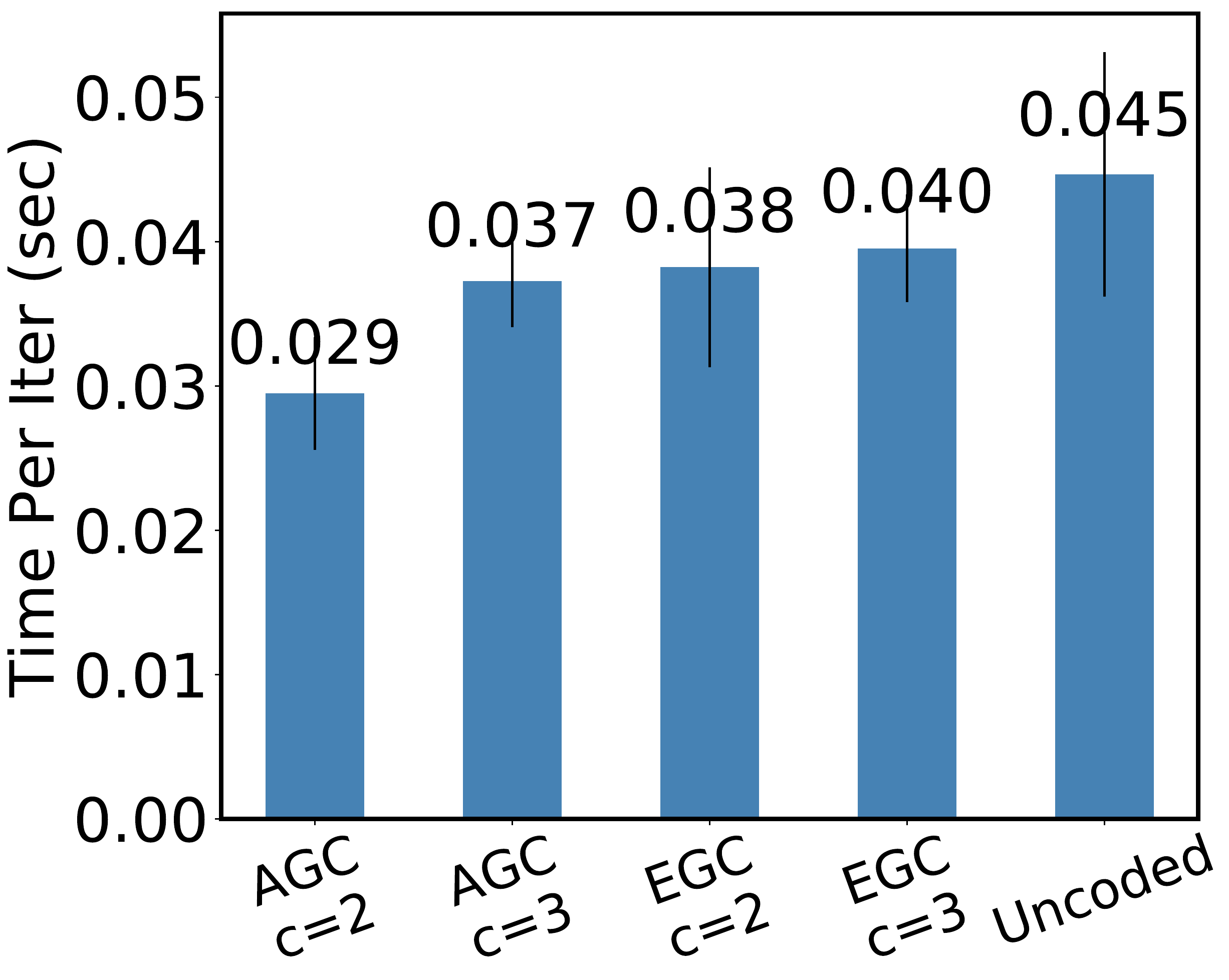}\label{fig:additional_results_e}}
	\subfigure[KC Housing]{\includegraphics[width=0.31\linewidth]{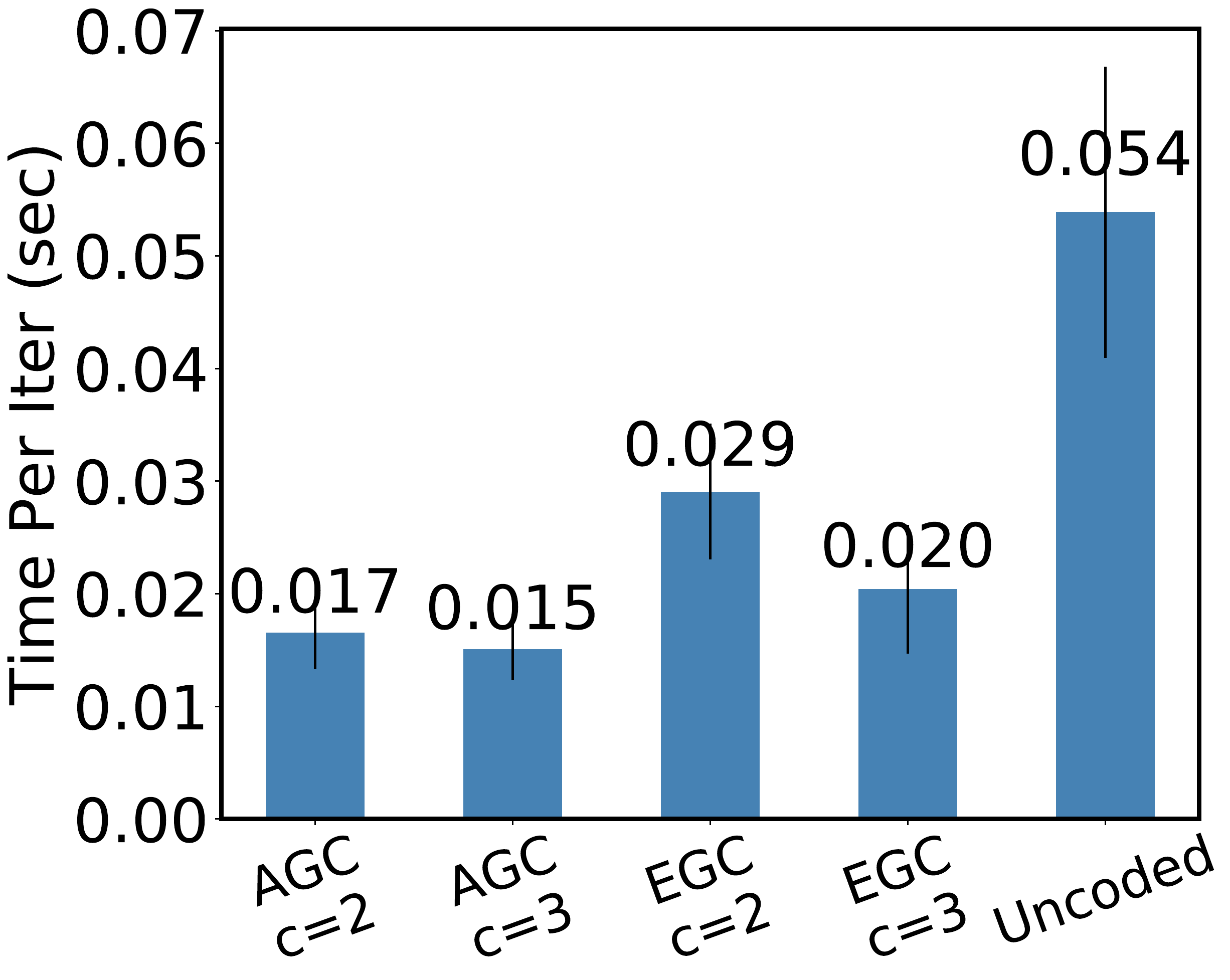}\label{fig:additional_results_f}}	
	\vspace{-0.3cm}
	\caption{Results on real world datasets without artificial simulated stragglers: (a) convergence performance on Amazon dataset, (b) convergence performance on Covertype dataset, (c) convergence performance on KC Housing dataset, (d) per iteration runtime on the Amazon dataset, (e) per iteration runtime on the Covertype dataset, (f) per iteration runtime on the KC Housing dataset.}
	\label{fig:additional_results}
\end{figure*}

\subsection{Experimental Results without Simulated Stragglers}\label{additional_experiments}
The experimental results conducted from the same process introduced previously without simulated stragglers are showed in Figure \ref{fig:additional_results}. Although the speedups of \erasegd{} over EGCs and vanilla GD are not as significant as those under simulated artificial delays, we still see similar results to the setting with simulated delays. Thus, even without simulating the straggler effect, \erasegd{} still performs at least as well and often better than distributed GD and distributed coded GD. Under practical scenarios where communication overheads are extremely heavy, this indicates that \erasegd{} has the potential to achieve large speedup gains over the other methods. 
	\section{Conclusion}

In this paper we present \erasegd{}, an new approach for distributed gradient descent that mitigates system delays using approximate gradient codes. \erasegd{} allows for erasures in the gradient computation in order to reduce the total training time. We show theoretically and empirically that \erasegd{} is capable of substantial improvements in distributed gradient descent.

While our theory focuses on a functions satisfying the Polyak-\L{}ojasiewicz condition, \erasegd{} can be directly implemented on much more general non-convex objectives, such as those encountered in distributed training of neural networks. We expect that in such situations it could still lead to substantial speedups, both theoretically and experimentally. We leave this analysis and empirical study for future study. Another interesting direction is combining \erasegd{} with variance reduction techniques as in SVRG~\cite{johnson2013accelerating}. This has the potential to improve the convergence rate of \erasegd{} with only a minimal increase in the average computation time per iteration.
	
	\section*{Acknowledgement}\label{sec:acknowledge}
	This work was partially supported by the MADLab AF Center of Excellence FA9550-18-1-0166, NSF's TRIPODS initiative: Institute for Foundations of Data Science at UW-Madison CCF-1740707, and AWS Cloud Credits for Research from Amazon.
	
	\bibliographystyle{plainnat}
	\bibliography{agc_conv}
	
	
	
\end{document}